\documentclass{article}

\usepackage[preprint]{neurips_2026}

\usepackage[utf8]{inputenc}
\usepackage[T1]{fontenc}
\usepackage{hyperref}
\usepackage{url}
\usepackage{booktabs}
\usepackage{amsfonts}
\usepackage{nicefrac}
\usepackage{microtype}
\usepackage{xcolor}
\usepackage{amsmath}
\usepackage{amssymb}
\usepackage{amsthm}
\usepackage{graphicx}
\usepackage{algorithm}
\usepackage{algorithmic}
\usepackage{cleveref}
\usepackage{dsfont}
\usepackage{wrapfig}
\usepackage{multirow}
\usepackage{enumitem}
\usepackage{caption}
\usepackage{subcaption}
\usepackage{booktabs}
\usepackage{makecell}

\theoremstyle{plain}      % gras + italique (pour théorèmes, corollaires)
\newtheorem{theorem}{Theorem}

\newtheorem{lemma}[theorem]{Lemma}
\newtheorem{proposition}[theorem]{Proposition}

\theoremstyle{definition} % gras + texte droit (pour définitions, assumptions)
\newtheorem{assumption}{Assumption}

\theoremstyle{remark}     % italique léger (pour remarques)
\newtheorem{remark}{Remark}

\title{LiST: Lipschitz Scaling Training\\
for Robust and Calibrated Neural Networks}

\author{
Arthur Chiron \\
IRIT \\
Toulouse, France \\
\And 
Franck Mamalet \\
IRT Saint Exupéry \\
Toulouse, France \\
\And 
Thomas Massena \\
IRIT, SNCF \\
Toulouse, France \\
\And
Thomas Deltort \\
IRIT, Airbus \\
Toulouse, France \\
\And
Mathieu Serrurier \\
IRIT \\
Toulouse, France \\
}

\begin{document}
    \maketitle

    \begin{abstract}
        While accuracy, robustness, and calibration are all essential for reliable neural networks, they are often studied separately; developing models that satisfy all three simultaneously remains a central challenge. Lipschitz-constrained models guarantee robustness by design, yet the manual selection of the Lipschitz constraint $L$ governs the resulting accuracy-robustness trade-off, and their calibration properties remain largely underexplored. In this work, we highlight a theoretical and empirical link between the enforced Lipschitz constraint and Temperature Scaling, a state-of-the-art calibration method. Specifically, we find that for a given training scheme, there exists a non-trivial value $L^*$ that yields an out-of-the-box calibrated network, and that calibration acts as a principled criterion to select a well-defined operating point on the accuracy-robustness Pareto front. Leveraging these insights, we introduce Lipschitz Scaling Training (LiST), a novel training paradigm that iteratively adjusts the global Lipschitz constant to reach this operating point. Through a margin parameter in the training loss, LiST further enables the construction of a fully calibrated Pareto front, allowing users to navigate the accuracy-robustness trade-off while remaining calibrated throughout. At convergence, LiST also enables the reintegration of calibration data into training, improving sample efficiency without sacrificing calibration. We validate LiST on CIFAR-10/100 and Tiny-ImageNet, demonstrating competitive accuracy and robustness against constrained and unconstrained baselines, while remaining calibrated out of the box. Code is available at~\href{https://anonymous.4open.science/r/LiST-E412}{GitHub}.
    \end{abstract}

    \section{Introduction}
    \label{sec:introduction}
    
    As deep learning models are increasingly deployed in safety-critical domains, from autonomous driving to medical diagnosis, ensuring their trustworthiness is paramount. This objective relies on pillars such as \textit{explainability}~\citep{ribeiro_why_2016, lundberg_unified_2017}, \textit{fairness}~\citep{dwork_fairness_2012, hardt_equality_2016}, and \textit{reliability}. The latter requires effective uncertainty quantification, whether through aligning predicted confidence with empirical accuracy, namely \textit{calibration}~\citep{brier_verification_1950, guo_calibration_2017}, or by providing formal coverage guarantees via \textit{conformal prediction}~\citep{vovk_algorithmic_2022, angelopoulos_conformal_2023}, as well as \textit{robustness}~\citep{szegedy2014intriguing, goodfellow2015explaining} to ensure stability against input perturbations. However, standard neural networks remain notoriously brittle and overconfident; they blindly produce high-confidence predictions even when faced with out-of-distribution or adversarial inputs.

    To address these limitations, research has followed, among others, two largely independent paths. On one hand, calibration methods seek to align confidence with accuracy, either through post-hoc rescaling like Temperature Scaling (TS)~\citep{guo_calibration_2017} or Dirichlet calibration~\citep{kull_beyond_2019}, or via specific losses~\citep{mukhoti_calibrating_2020}. While effective at reducing overconfidence, these techniques generally do not incentivize robustness. Independently, Lipschitz-constrained Neural Networks (LNNs)~\citep{anil_sorting_2019, serrurier_achieving_2021} provide deterministic robustness guarantees by enforcing an upper bound on the network's global Lipschitz constant. While this constant theoretically dictates the network's position on the accuracy-robustness Pareto frontier, it is usually fixed to 1 in practice to satisfy the product bound rule of spectrally normalized or orthogonalized layers. Instead, practitioners rely on a temperature parameter within the loss function~\citep{bethune_pay_2022} to navigate this trade-off. However, this parameter is typically treated as a static value and has to be tuned, without leveraging the model's internal calibration as a guiding signal.

    In this work, we argue that calibration (through Temperature Scaling) and robustness could be handled together. We highlight the structural duality between the temperature parameter and the Lipschitz constant, showing that they have an equivalent effect on the network's logits, albeit through fundamentally different mechanisms. Since Temperature Scaling seeks a scalar $T^*$ that improves the calibration of a given network by rescaling its logits, this duality suggests the existence of an $L^*$ that achieves the same effect intrinsically. Based on this insight, we hypothesize that the ``calibrated'' Lipschitz constant $L^*$ corresponds to a principled equilibrium on the accuracy-robustness Pareto front. To reach this operating point without manual tuning, we propose \textbf{LiST (Lipschitz Scaling Training)}. LiST is a dynamic training strategy that leverages the temperature scaling factor, computed on a calibration set at each epoch, as a feedback signal to adaptively tune the network's Lipschitz constraint. By converging towards this intrinsic constant, LiST enables a unique capability: yielding models that are simultaneously certifiably robust and calibrated out-of-the-box. Furthermore, by varying an offset parameter in the training loss, LiST enables to recover a fully calibrated Pareto front, allowing users to navigate the accuracy-robustness trade-off while remaining calibrated throughout. Finally, once this constant is identified and fixed, freezing the Lipschitz constraint strictly bounds the model's capacity, making the re-integration of calibration data into the training process theoretically safe and practically beneficial for sample efficiency, without sacrificing calibration on the test set.

    Our contributions are summarized as follows:
    \begin{itemize}
        \item We formalize the structural duality between Temperature Scaling and Lipschitz constraints, and show that calibration acts as a principled, unsupervised criterion to select a natural operating point on the accuracy-robustness Pareto front.
        \item We introduce LiST, a dynamic training algorithm that automatically identifies $L^*$ by using the temperature scaling factor as a feedback signal, yielding models that are simultaneously calibrated and certifiably robust by construction.
        \item We show that LiST enables the construction of a fully calibrated Pareto front by varying an offset parameter in the training loss, decoupling calibration from the accuracy-robustness trade-off.
        \item We empirically demonstrate that freezing the constraint at $L^*$ makes the re-integration of calibration data into training safe and beneficial across CIFAR-10/100 and Tiny-ImageNet, improving sample efficiency without sacrificing calibration.
    \end{itemize}

    \section{Background and Related Work}
    \label{sec:background_and_related_work}

    We consider a neural network $f: \mathcal{X}\to \mathbb{R}^{C}$ mapping an input $\mathbf{x} \in \mathcal{X} \subseteq \mathbb{R}^{d}$ to a vector of logits $\mathbf{z} = f(\mathbf{x})$. The true label is denoted $y$ (or $\mathbf{y}$ for its one-hot encoding), while the predicted class is $\hat{y}= \arg\max\mathbf{z}$, and the associated confidence vector is given by $\hat{\mathbf{p}}= \mathrm{softmax}(\mathbf{z})$.

    \subsection{Lipschitz Constrained Neural Networks and Adversarial Robustness}
    \label{subsec:lipschitz_constrained_neural_networks_and_adversarial_robustness}

    The function $f$ is said to be globally $L$-Lipschitz continuous with respect to the $\ell_{2}$-norm if for all $\mathbf{x}, \mathbf{x}' \in \mathcal{X}$:
    \begin{equation}
        \|f(\mathbf{x}) - f(\mathbf{x}')\|_{2}\le L \|\mathbf{x} - \mathbf{x}'\|_{2}.
    \end{equation}
    In practice, computing the exact Lipschitz constant of a neural network is known to be an NP-hard problem~\citep{virmaux_lipschitz_2018}. Therefore, for a feed-forward network composed of $D$ layers $f = \phi_{D}\circ \dots \circ \phi_{1}$, the global Lipschitz constant $L$ is typically controlled via its upper bound, defined as the product of the Lipschitz constants of all constituent layers. This encompasses linear transformations (bounded by their spectral norms) as well as non-linear operations such as activations, pooling, and normalization layers. Consequently, if each layer $\phi_{i}$ is constrained to be $L_{i}$-Lipschitz, the global constant satisfies the inequality $L \le \prod_{i=1}^{D}L_{i}$.
    
    The primary motivation for constraining $L$ is to both \textit{improve} and \textit{guarantee} robustness against adversarial examples~\citep{szegedy2014intriguing, goodfellow2015explaining}: small perturbations of the input that cause the classifier to make incorrect predictions. Formally, an adversarial example is a perturbed input $\mathbf{x} + \boldsymbol{\delta}$ such that $\arg\max f(\mathbf{x} + \boldsymbol{\delta}) \neq y$ and $\|\boldsymbol{\delta}\|_2 \leq \varepsilon$ for some perturbation budget $\varepsilon > 0$. This has motivated two complementary approaches to approximate the true \textit{Robust Accuracy} $\mathrm{RA}_\varepsilon$, defined via the largest budget $\varepsilon$ under which the classifier remains correct:
    \begin{equation}
        \begin{gathered}
            R_{\mathrm{true}}(\mathbf{x}) = \sup \left\{ \varepsilon \geq 0 : \arg\max f(\tilde{\mathbf{x}}) = y,\ \forall \tilde{\mathbf{x}} : \|\tilde{\mathbf{x}} - \mathbf{x}\|_2 \leq \varepsilon \right\}, \\
            \mathrm{RA}_\varepsilon(f) = \frac{1}{N}\sum_{i=1}^{N} \mathds{1}\left[ R_{\mathrm{true}}(\mathbf{x}_i) \geq \varepsilon \right].
        \end{gathered}
    \end{equation}
    Since computing $R_{\mathrm{true}}$ exactly is NP-hard, attack algorithms such as FGSM~\citep{goodfellow2015explaining}, PGD~\citep{madry_towards_2018}, or AutoAttack~\citep{croce_reliable_2020} are used in practice to search for worst-case perturbations within the budget. Since these attacks may fail to find the true worst case, the resulting \textit{Empirical Robust Accuracy} $\mathrm{ERA}_\varepsilon$ provides an upper bound on $\mathrm{RA}_\varepsilon$.
    
    For $L$-Lipschitz classifiers, the \textit{certified robust radius} $R_{\mathrm{cert}}(\mathbf{x})$ provides instead a provable lower bound on $R_{\mathrm{true}}(\mathbf{x})$~\citep{tsuzuku_lipschitz-margin_2018}, and the \textit{Certified Robust Accuracy} $\mathrm{CRA}_\varepsilon$ is defined symmetrically:
    \begin{equation}
        R_{\mathrm{cert}}(\mathbf{x}) = \mathds{1}_{\hat{y}=y} \cdot \frac{\mathcal{M}(\mathbf{x})}{\sqrt{2}L}, \quad \mathrm{CRA}_\varepsilon(f) = \frac{1}{N}\sum_{i=1}^{N} \mathds{1}\left[ R_{\mathrm{cert}}(\mathbf{x}_i) \geq \varepsilon \right],
    \end{equation}
    with $\mathcal{M}(\mathbf{x}) = z_y - \max_{j \neq y} z_j$ the margin at point $\mathbf{x}$. Together, these two quantities bracket the true robust accuracy: $\mathrm{CRA}_\varepsilon \leq \mathrm{RA}_\varepsilon \leq \mathrm{ERA}_\varepsilon$, making them complementary tools for evaluating robustness. The expression of $R_{\mathrm{cert}}$ further reveals a fundamental trade-off: maximizing it requires simultaneously increasing the margins and minimizing $L$. Historically, certified defenses have prioritized structural guarantees by strictly fixing $L=1$ via spectral normalization or orthogonalization~\citep{anil_sorting_2019, bethune_pay_2022, boissin2025adaptive}. However, this rigid constraint often hinders the network's expressivity, preventing it from producing the large margins required for certification. To overcome this bottleneck, recent works~\citep{bethune_pay_2022, prach_almost-orthogonal_2023, prach2025intriguing, lai2025enhancing} introduce a temperature parameter $\tau$ and an offset $\xi$ in the training loss:
    \begin{equation}
        \mathcal{L}(\mathbf{z},\mathbf{y}) = \mathrm{cross\_entropy}\left(\mathrm{softmax}\left(\frac{\mathbf{z} - \xi \mathbf{y}}{\tau}\right),\ \mathbf{y}\right).
    \end{equation}
    We adopt this loss in our work, with one key difference: rather than treating $\tau$ as a tunable hyperparameter, we absorb it directly into the Lipschitz constant $L$ of $f$, since dividing the logits by $\tau$ is equivalent to scaling $L$ by $1/\tau$. This makes the implicit effect of $\tau$ on the Lipschitz constant explicit and transparent.

    \subsection{Model Calibration}
    \label{subsec:model_calibration}

    While accuracy measures the frequency of correct predictions, \textit{calibration} measures the reliability of the confidence scores. A model is perfectly calibrated if its predicted probability reflects its true correctness likelihood. Formally:
    \begin{equation}
        \forall p \in [0,1],\ \mathbb{P}(\hat{y} = y \mid \max \hat{\mathbf{p}} = p) = p.
    \end{equation}

    The study of calibration has a long history:~\citep{brier_verification_1950} introduced the first proper scoring rule to measure probabilistic forecast accuracy, and~\citep{murphy_reliability_1977} proposed Reliability Diagrams as a visual tool to assess calibration quality. In the deep learning era, the Expected Calibration Error (ECE)~\citep{naeini_obtaining_2015} has become the standard metric. It partitions predictions into $M$ bins based on confidence and computes the weighted average absolute difference between accuracy and confidence in each bin:
    \begin{equation}
        \text{ECE}= \sum_{m=1}^{M}\frac{|B_{m}|}{N}\left| \text{acc}(B_{m}) - \text{conf}(B_{m}) \right|.
    \end{equation}
    While ECE measures the overall magnitude of miscalibration, it does not indicate its direction. The Expected Signed Calibration Error (ESCE)~\citep{verhaeghe_generalizable_2023, ledda_robustness_2025} addresses this by removing the absolute value, yielding a signed metric where negative values indicate underconfidence and positive values indicate overconfidence. In practice, miscalibration tends to be systematic rather than sample-specific, making sign cancellation across bins rare.

    Several families of methods have been proposed to improve calibration. Training-time approaches modify the loss function directly, such as label smoothing~\citep{muller_when_2019} or Focal Loss~\citep{mukhoti_calibrating_2020}. Post-hoc methods instead recalibrate a fixed model on a held-out calibration set $\mathcal{D}_\mathrm{cal}$, disjoint from the training data. Among these, Temperature Scaling (TS)~\citep{guo_calibration_2017} is the most widely adopted: it introduces a single scalar $T > 0$ to rescale the logits before the softmax:
    \begin{equation}
        \mathbf{\hat{p}} = \mathrm{softmax}(\mathbf{z}/T),
    \end{equation}
    where $T$ is optimized by minimizing the negative log-likelihood on $\mathcal{D}_\mathrm{cal}$. Extensions of TS such as Vector Scaling and Dirichlet calibration~\citep{kull_beyond_2019} apply class-wise or full covariance rescaling, at the cost of additional parameters and increased risk of overfitting on $\mathcal{D}_\mathrm{cal}$.

    In the context of this work, we define an intrinsically calibrated network as one for which the optimal temperature scaling factor is one ($T^* = 1$), meaning no post-hoc correction is needed (regarding TS). Any deviation $T^* \neq 1$ is interpreted as a manifestation of miscalibration. As we will show in~\Cref{sec:the_lipschitz_temperature_duality}, the Lipschitz constant $L$ plays a structural role in determining this optimal temperature, which motivates our approach.

    \section{The Lipschitz-Temperature duality}
    \label{sec:the_lipschitz_temperature_duality}

    In this section, we formalize the relationship between the global Lipschitz constant $L$ of a network and the temperature parameter $T$ used in Temperature Scaling (TS). We highlight that these quantities are functionally coupled, suggesting that the search for an optimal calibration temperature can be reformulated as a search for an optimal Lipschitz constant.

    \subsection{Link between Temperature Scaling \texorpdfstring{$T$}{T} and Lipschitz constant \texorpdfstring{$L$}{L}}
    \label{subsec:link_between_t_and_l}

    Consider a neural network $f$ whose constituent layers are each Lipschitz continuous (as in~\Cref{subsec:lipschitz_constrained_neural_networks_and_adversarial_robustness}), so that $f$ admits a finite global Lipschitz constant $L$. Standard TS yields a calibrated network $\Tilde{f}(x) = f(x)/T$ which is trivially $(L/T)$-Lipschitz. Thus, TS can be viewed as a post-hoc rescaling of the constant. Conversely, changing the Lipschitz constant of the network $f$ to a target Lipschitz constant $K$ can be achieved by dividing $f$ by a factor $T = L/K$, which can be viewed as form of TS. Consequently, $T$ and $L$ act as interchangeable scaling factors on the logits. Based on this duality, we hypothesize the existence of an optimal Lipschitz constant $L^{*}$ that yields an intrinsically calibrated network.
    
    We empirically validate the equivalence derived above by analyzing the sensitivity of the ECE to changes in logit scaling mechanisms. \Cref{fig:correlation} illustrates this parallel by contrasting the ECE profile of (a) networks trained with varying Lipschitz constants against (b) a fixed 1-Lipschitz network subjected to post-hoc scaling. As observed, both mechanisms yield a qualitatively similar profile. In both cases, a low scaling factor results in under-confidence, while a high factor leads to over-confidence (see~Appendix~\ref{app:additional_results_and_qualitative_analysis}). However, we observe a quantitative mismatch: the optimal training constraint ($L^{*}\approx 32$) differs significantly from the optimal post-hoc scale applied to the 1-Lipschitz baseline ($1/T^{*}\approx 1.85$). This mismatch is mainly due to the fact that the 1-Lipschitz network is much less accurate than the 32-Lipschitz one, which likely yields a smoother calibrated posterior and thus reduces the need for a large temperature correction. This dynamic between the optimal network with respect to the Lipschitz constant and Temperature Scaling will be the core of our algorithm.
    \begin{figure}[t]
        \centering
        \includegraphics[width=\linewidth]{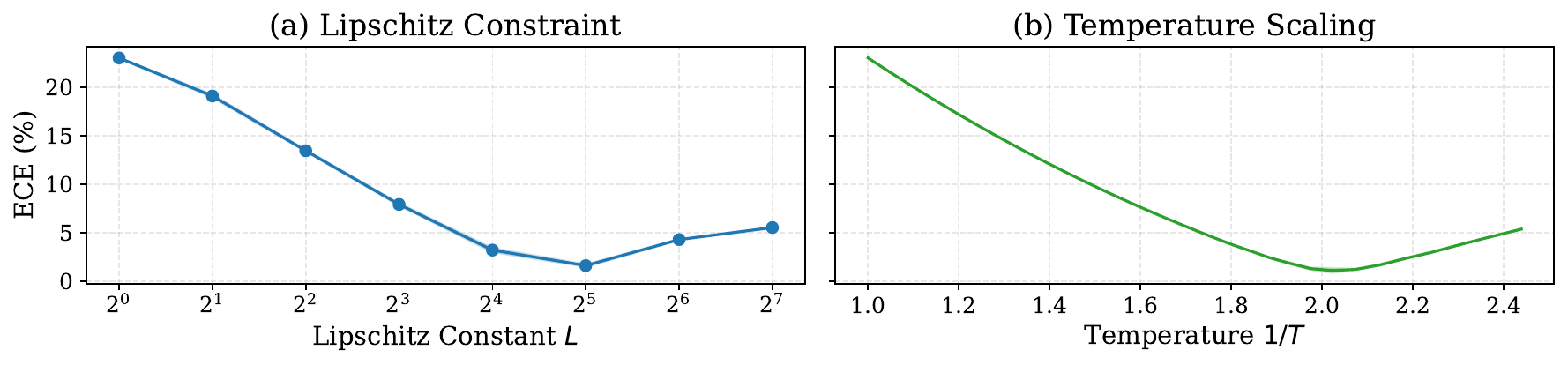}
        \caption{
        \textbf{ECE as a function of the logit scaling mechanism.} (a) Networks trained with varying fixed Lipschitz constants $L$. (b) Post-hoc Temperature Scaling on a fixed 1-Lipschitz network (plotted against $1/T$). Both profiles are qualitatively similar.
        % \textbf{Empirical correlation between T and L.} We report the test ECE for two scenarios: (a) Training independent networks with different fixed Lipschitz constants $L$. (b) Varying the post-hoc temperature $T$ on a fixed 1-Lipschitz network (plotted against $1/T$). Both mechanisms exhibit a qualitatively similar profile.
        }
        \label{fig:correlation}
    \end{figure}
    \subsection{On the Optimal Lipschitz Constant}
    \label{subsec:on_the_optimal_lipschitz_constant}

    As observed in the experiment of~\Cref{fig:correlation}, where an optimal value of $L$ naturally emerges, we formulate the hypothesis that there exists a specific training constraint that achieves near-optimal calibration naturally, comparable to the performance of TS on a given network. We formalize this existence result in~Proposition~\ref{prop:existence_L_star} (Appendix~\ref{app:existence_of_the_optimal_lipschitz_constant}), which establishes that under mild regularity assumptions on the training dynamics, there always exists an intrinsic $L^*$ such that $T^*(L^*) = 1$.

    However, determining \textit{a priori} which $L^*$ will yield this out-of-the-box calibrated network is non-trivial, as it depends on the architecture, dataset, and training dynamics. Ideally, knowing this constant beforehand would eliminate the need for a calibration set (typically used for post-hoc calibration), a significant advantage in low-data regimes. A naive strategy might be to train a reference $1$-Lipschitz network to compute its optimal temperature $T^{*}$, then retrain a new model with $L' = 1/T^{*}$. Yet, this heuristic fails due to the mismatch observed above: constraining $L$ fundamentally alters the optimization trajectory and the learned features~\citep{bethune_pay_2022}, which prevents the naively derived constant from being optimal.

    Furthermore, imposing a low Lipschitz constant theoretically bounds the confidence of the model, even at optimality. While minimizing the cross-entropy loss aims to align the network output $\hat{p}$ with the true posterior distribution $p(y|x)$, the Lipschitz constraint restricts the slope and magnitude of the learned function. If the true posterior contains sharp transitions, corresponding to regions of high certainty separated by narrow margins, a sufficiently low $L$ physically prevents the network from fitting these variations. Instead, the network yields a smoothed approximation of the posterior, where the probability mass is spread more evenly across classes. Consequently, an overly constrained network is structurally forced to exhibit higher entropy than the ground truth, resulting in unavoidable under-confidence.

    \subsection{The Robustness-Accuracy-Calibration trade-off}
    \label{subsec:the_robustness_accuracy_calibration_trade_off}

    Although $T$ and $L$ have similar impact on calibration, they differ fundamentally in their effect on robustness. Temperature Scaling is a post-hoc operation that does not alter the feature space geometry. As shown in~\Cref{eq:invariance}, applying TS scales both the margin $\mathcal{M}$ and the Lipschitz constant $L$ by the same factor $1/T$, leaving the certified radius $R_\mathrm{cert}$ invariant:
    \begin{equation}\label{eq:invariance}
        R_{\mathrm{cert}}^{\Tilde{f}}(\mathbf{x}) = \frac{\mathcal{M}(f(\mathbf{x})/T)}{\sqrt{2}(L/T)} = \frac{\mathcal{M}(f(\mathbf{x}))}{\sqrt{2}L} = R_{\mathrm{cert}}^{f}(\mathbf{x}).
    \end{equation}
    Consequently, TS can calibrate a model but cannot improve its robustness. To achieve both, one must actively constrain $L$ during optimization.
    
    Constraining $L$ allows us to navigate the Accuracy-Robustness Pareto front, where a low $L$ favors robustness at the expense of accuracy, while a high $L$ maximizes performance but increases sensitivity to perturbations~\citep{bethune_pay_2022}. We argue that the ``calibrated" Lipschitz constant $L^{*}$ (where $T \simeq 1$) does not lie arbitrarily on this frontier. Instead, calibration acts as an unsupervised criterion to identify a principled equilibrium between under-fitting (over-constrained) and instability (under-constrained). As we will illustrate in the experiments (\Cref{fig:pareto_analysis}a), this criterion selects a non-trivial operating point without requiring manual tuning of the trade-off parameter $L$.

    In summary, our analysis supports the shift towards a dynamic training strategy. First, the functional duality between $L$ and $T$ implies that calibration can be achieved by optimizing the Lipschitz constant directly. Second, unlike Temperature Scaling, constraining $L$ yields certified robustness, making it a superior control variable. Finally, since the optimal constant $L^{*}$ is intractable \textit{a priori} due to its impact on feature learning, it cannot be fixed statically. These findings motivate our proposed algorithm, detailed in the next section, which dynamically targets this optimal constant to simultaneously achieve calibration and certified robustness.

    \section{LiST: Lipschitz Scaling Training}
    \label{sec:list_lipschitz_scaling_training}

    Building on the duality established in Section~\ref{sec:the_lipschitz_temperature_duality}, LiST automates the search for $L^*$ through a two-phase training procedure.

    \subsection{Algorithm Overview}
    \label{subsec:algorithm_overview}
    
    \paragraph{Phase I: Dynamic Search.}
    Phase I treats the Lipschitz constraint as a dynamic variable rather than a fixed hyperparameter. At each epoch, after a standard weight update, the calibration temperature $T_e^*$ is computed on a held-out calibration set $\mathcal{D}_{\text{cal}}$ and the constraint is updated according to the rule: 
    \begin{equation}
        L_{e+1} \leftarrow L_e / T^*_e.
    \end{equation}
    This rule implements a negative feedback loop: over-confidence ($T_e^* > 1$) tightens the constraint, while under-confidence ($T_e^* < 1$) relaxes it. The system naturally converges to the equilibrium $T^* = 1$, at which point the constraint has stabilized at the intrinsic value $L^*$, a principled operating point on the accuracy-robustness Pareto front, selected without any manual tuning.
    
    \paragraph{Phase II: Data Re-integration.}
    Phase II exploits this convergence. With the constraint frozen at $L^*$, the calibration data $\mathcal{D}_{\mathrm{cal}}$, no longer needed as a feedback signal, is re-integrated into training: the network is fine-tuned on $\mathcal{D}_{\mathrm{total}} = \mathcal{D}_{\mathrm{train}} \cup \mathcal{D}_{\mathrm{cal}}$. The theoretical justification for this re-integration is discussed in Section~\ref{subsec:stopping_criterion_and_data_reintegration}.
    
    Detailed pseudocode for both phases is provided in~Appendix~\ref{app:algorithm}.

    \subsection{Convergence Dynamics}
    \label{subsec:convergence_dynamics}
    
    The dynamics of Phase~I, illustrated in~\Cref{fig:training_dynamics}, are characterized by two successive regimes driven by the interplay between the Lipschitz constraint and the calibration temperature.
    \begin{figure}[t]
        \centering
        \includegraphics[width=\linewidth]{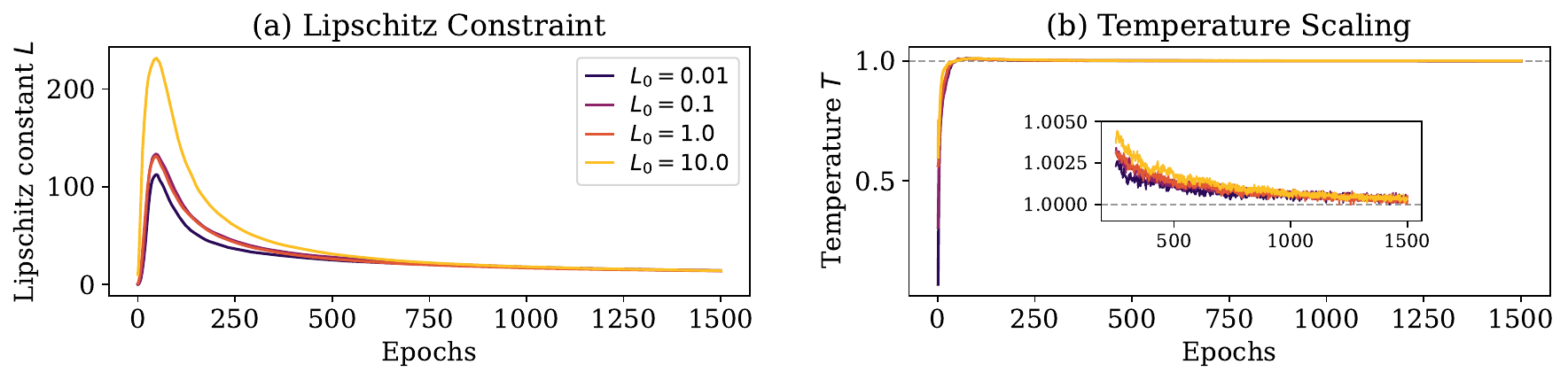}
        \caption{
        %\textbf{Training dynamics of LiST (Phase~I).} (a) Evolution of $L$ for four initializations $L_0$, converging to the same $L^*$. (b) Corresponding calibration temperature $T$, converging to $T=1$.
        \textbf{Training dynamics of LiST (Phase~I).} (a) Evolution of the Lipschitz constraint $L$ for four initializations $L_0$. Regardless of $L_0$, the algorithm first relaxes the constraint to gain expressivity (\textit{relaxation regime}), then progressively tightens it as the model becomes over-confident (\textit{contraction regime}), converging to the same intrinsic value $L^*$. (b) Corresponding evolution of the calibration temperature $T$. The inset highlights the asymptotic convergence towards the target equilibrium $T = 1$.
        }
        \label{fig:training_dynamics}
    \end{figure}
    At initialization, randomly initialized weights produce high-entropy, low-confidence predictions ($T < 1$). The feedback rule $L_{e+1} \leftarrow L_e / T^*_e$ interprets this under-confidence as insufficient model capacity and responds by relaxing the constraint, causing $L_e$ to increase rapidly. This \textit{relaxation regime} acts as an automatic structural warm-up: the network progressively acquires the expressivity needed to fit the training data without any manual capacity scheduling. Once the model begins to learn discriminative features, predictions sharpen and over-confidence emerges ($T > 1$), triggering the \textit{contraction regime}: the constraint is tightened at each epoch ($L_{e+1} < L_e$), reducing model capacity until the system reaches the calibrated equilibrium $T \approx 1$. At convergence, the constraint has stabilized at $L^*$, the tightest value compatible with calibration.
    
    Beyond this convergence behavior, \Cref{fig:training_dynamics}a shows that the feedback loop is robust to the choice of $L_0$. Despite initializations spanning three orders of magnitude ($L_0 \in \{0.01, 0.1, 1.0, 10.0\}$), the final constraints converge to a narrow range, with a relative spread of less than $3\%$ on CIFAR-10. This confirms that $L^*$ is an intrinsic property of the task and architecture rather than a hyperparameter, effectively eliminating the grid search over $L$ required by static Lipschitz baselines.
    
    % A direct consequence of this non-stationary dynamics is that standard learning rate schedulers (e.g., Cosine Decay) become ineffective: since $L$ directly scales the network's activations and gradients throughout both regimes, a fixed schedule may decay the learning rate prematurely or fail to adapt to the evolving gradient scale. We therefore employ the ScheduleFree optimizer~\citep{defazio_road_2024}, which removes the need for a rigid schedule and adapts robustly to the changing scale of the problem, ensuring stability from the rapid expansion of the relaxation regime to the fine-grained adjustments of the contraction regime.

    \subsection{Stopping Criterion and Data Re-Integration}
    \label{subsec:stopping_criterion_and_data_reintegration}

    Defining an optimal stopping criterion for LiST requires addressing specific dynamics absent in standard training. First, relying on a fixed epoch budget is prone to sub-optimality: since CRA relies on the ratio between margins and the Lipschitz constant, and LiST refines $L$ asymptotically, an arbitrary cutoff often terminates training before the maximal certified radius is achieved. Second, monitoring validation loss is counter-productive; cross-entropy minimizes prediction error but does not explicitly incentivize the large margins required for robustness. Consequently, the ``robustification'' phase often triggers a stagnation or slight increase in validation loss, leading to premature stopping. Third, a criterion based solely on temperature convergence (i.e., stopping when $T_{t} \approx 1$) is inconsistent due to scale invariance: a small deviation (e.g., $T=0.99$) induces a negligible update when $L$ is small, but a massive shift when $L$ is large. Therefore, we adopt the stability of the Lipschitz constant as the most coherent criterion. Specifically, training is halted when the relative range of $L_{e}$ over a sliding window $\mathcal{W}=\left[L_t\right]_{e-W< t \leq e}$ of size $W$ falls below a tolerance $\varepsilon$ (i.e., $(\max(\mathcal{W}) - \min(\mathcal{W})) / \text{mean}(\mathcal{W}) < \varepsilon$), ensuring that the network has reached its true structural equilibrium.

    Identifying this equilibrium also resolves the data-splitting limitation. Once the constraint is frozen at $L^*$, the network remains in the function class $\mathcal{F}_{L^*}$ of $L^*$-Lipschitz functions throughout fine-tuning. Standard generalization bounds for norm-controlled deep networks \citep{bartlett_spectrally-normalized_2017, neyshabur2018pac, golowich2018size} establish that the Rademacher complexity of such classes scales as $\mathcal{O}(L^* / \sqrt{N})$, where $N$ is the sample size. Re-integrating $\mathcal{D}_{\mathrm{cal}}$ into training therefore strictly tightens the generalization gap by a factor $\sqrt{n / (n + m)}$ (where $n = |\mathcal{D}_{\mathrm{train}}|$ and $m = |\mathcal{D}_{\mathrm{cal}}|$), without enlarging the underlying capacity class. While this argument does not by itself guarantee improved test performance, it ensures that Phase~II cannot harm generalization: any improvement observed empirically (Table~\ref{tab:low_data_and_datasets}, bottom) reflects a tighter bound combined with the additional supervision signal provided by $\mathcal{D}_{\mathrm{cal}}$.
    % Once $L^*$ is reached, the network's capacity is strictly bounded. While overfitting remains theoretically possible, the number of epochs required to memorize data grows exponentially as capacity is restricted, effectively neutralizing the risk within a practical training horizon. Since the calibration and training sets are drawn i.i.d from the same underlying distribution (\Cref{ass:iid}), we safely introduce a final recovery phase (\textbf{Phase~II}): the constraint is \textbf{frozen} ($L \leftarrow L^*$) and the network is fine-tuned on the combined dataset. Fixing $L$ prevents structural drift, and~\Cref{thm:phase2} formally guarantees that this re-integration strictly tightens the generalization bound by a factor $\sqrt{n/N}$, while the capacity class $\mathcal{F}_{L^*}$ remains unchanged.

    \section{Experiments}
    \label{sec:experiments}

        \subsection{Experimental Setup}
        \label{subsec:experimental_setup}
    
        We compare LiST against two families of baselines. For unconstrained networks, we consider standard cross-entropy training as a reference for clean accuracy, Label Smoothing~\citep{muller_when_2019} and Focal Loss~\citep{mukhoti_calibrating_2020} as methods targeting overconfidence reduction, and PGD-AT~\citep{madry_towards_2018} and CAAT~\citep{obadinma_calibration_2024} as adversarially robust training procedures. For Lipschitz-constrained networks, we compare against fixed-$L$ baselines with $L \in \{1, 2, 4, 8, 16, 32, 64, 128\}$, collectively tracing the accuracy-robustness Pareto front. Rather than positioning LiST as a direct competitor to any of these methods, our goal is to show that it achieves competitive levels across all three axes simultaneously; clean accuracy, calibration, and robustness, and all baselines are therefore evaluated without post-hoc calibration, so as to assess the intrinsic calibration of each training method. CRA is not reported for unconstrained networks, as it requires a provable Lipschitz bound.
        
        We evaluate on CIFAR-10, CIFAR-100, and Tiny-ImageNet. LiST uses a stratified 90/10 train/calibration split, while all baselines train on the full training set. The official Tiny-ImageNet validation set serves as test set. All models use a ResNet-18 adapted to each dataset's resolution: unconstrained baselines use a standard ResNet-18, while Lipschitz-constrained models share the architecture detailed in~Appendix~\ref{app:implementation}. Full training details are provided in~Appendix~\ref{app:implementation}.
        
        LiST is initialized with $L_0 = 1$, which is justified by the convergence analysis of Section~\ref{subsec:convergence_dynamics} showing that the algorithm reliably reaches $L^*$ regardless of initialization. The Phase~I stopping criterion uses a sliding window of $W = 30$ epochs with tolerance $\varepsilon = 0.001$. Phase~II fine-tunes on $\mathcal{D}_{\mathrm{total}} = \mathcal{D}_{\mathrm{train}} \cup \mathcal{D}_{\mathrm{cal}}$ for 100 additional epochs across all datasets.
        
        We report clean accuracy, ECE and ESCE for calibration, AutoAttack accuracy~\citep{croce_reliable_2020} ($\ell_2$, $\varepsilon = 0.5$), and CRA ($\ell_2$, $\varepsilon = 36/255$), following the respective community standards. Calibration is also assessed qualitatively via Reliability Diagrams in~Appendix \ref{app:qualitative_analysis}. The choice of evaluation budgets is discussed in~Appendix~\ref{app:implementation}.
    
        % \franck{Proposition d organisation des results
        % \subsection{Robustness-Accuracy-Calibration trade-off}
        % Lower part (Lipschitz)  of \cref{tab:main_results} and \cref{fig:pareto_analysis} (left part), compare several fixed Lipschitz constraint $L \in \{1, 2, 4, 8, 16, 32, 64, 128\}$, with the proposed  solution LiST for $\xi=0$. This show that the LiST ...on the pareto front neither the most robust,....
    
        % Last line of \cref{tab:main_results} and \cref{fig:pareto_analysis} (left part), demonstrate that List decouples calibrattion.... withthe $\xi$ parameter...
        
        % \subsection{Comparison with previous work}
        % Top part of \cref{tab:main_results} compare LiST againt unconstrained ......
        % Ajouter si possible des chiffres pour Cifar100 et Tiny Imagenet
        
        % \subsection{LiST validity on low data regime}
        
        % Table 2 top part (mais il faut pas la partie bottom) 
        % \subsection{Ablation study}
        % Table 2 bottom part ds une table 3
        % }
        
        \subsection{Results}
        \label{subsec:results}
    
        Table~\ref{tab:main_results} compares LiST against unconstrained and Lipschitz-constrained baselines on CIFAR-10. The first key observation is that LiST achieves the strongest calibration across all methods, with an ECE of $0.79\%$, well below both unconstrained baselines (down to $1.68\%$ for Focal Loss) and the best fixed-$L$ baseline ($L=32$, ECE $1.60\%$). Crucially, this calibration is obtained out-of-the-box, without any post-hoc correction. The second observation concerns the position of LiST on the accuracy-robustness Pareto front: with $\xi = 0$, LiST converges to $L^* \simeq 34.66$, yielding a model that is neither the most accurate nor the most robust, but sits at a natural, principled equilibrium between the two objectives, as illustrated in Figure~\ref{fig:pareto_analysis}a. This confirms that calibration acts as an unsupervised criterion to select a non-trivial operating point, without any manual tuning of $L$. Finally, by increasing the loss offset to $\xi = 3$, LiST moves along the Pareto front towards higher robustness (AA Acc.\ $53.04\%$, CRA $63.06\%$), while remaining well-calibrated (ECE $1.17\%$). This demonstrates that LiST decouples calibration from the accuracy-robustness trade-off: the user can freely navigate this trade-off via $\xi$, while calibration is preserved throughout — as shown in Figure~\ref{fig:pareto_analysis}b.
        \begin{table*}[t]
            \centering 
            \footnotesize
            \caption{
            % \textbf{Main results on CIFAR-10.} Comparison of LiST against unconstrained and Lipschitz-constrained baselines. Results are $\text{mean}_{\pm \text{std}}$ over 3 runs; best in \textbf{bold}, second-best \underline{underlined}.
            \textbf{LiST improves calibration while remaining competitive in robustness on CIFAR-10.} We compare LiST to unconstrained baselines trained with different objectives and to fixed-$L$ Lipschitz models (subsampled). LiST achieves the strongest calibration overall, and tuning the offset $\xi$ moves the model along the accuracy-robustness Pareto front while preserving low calibration error. Results are reported as $\text{mean}_{\pm \text{std}}$ over 3 runs. Best are in \textbf{bold}; second-best are \underline{underlined}.
            }
            \label{tab:main_results}
            \tabcolsep=0.11cm
            \begin{tabular}{@{}clccccc@{}}
                \toprule 
                & \textbf{Method} & \textbf{Clean Acc.} ($\uparrow$) & \textbf{ECE} ($\downarrow$) & \textbf{ESCE} ($\rightarrow 0$) & \textbf{AA Acc.} ($\uparrow, \varepsilon = 0.5$) & \textbf{CRA} ($\uparrow, \varepsilon = 36/255$) \\
                \midrule 
                \multirow{5}{*}{\rotatebox{90}{Unconst.}}
                & Standard      & $92.95_{\pm 0.20}$ & $4.72_{\pm 0.24}$ & $4.70_{\pm 0.25}$ & $0.28_{\pm 0.01}$ & N/A \\
                & Label Smooth. & $\mathbf{93.11}_{\pm 0.15}$ & $6.00_{\pm 0.15}$ & $-4.81_{\pm 0.16}$ & $0.57_{\pm 0.16}$ & N/A \\
                & Focal Loss    & $\underline{93.08}_{\pm 0.12}$ & $1.68_{\pm 0.08}$ & $1.63_{\pm 0.10}$ & $0.13_{\pm 0.02}$ & N/A \\
                & PGD-AT        & $86.99_{\pm 0.22}$ & $7.67_{\pm 0.24}$ & $7.63_{\pm 0.25}$ & $\mathbf{61.88}_{\pm 0.14}$ & N/A \\
                & CAAT          & $88.83_{\pm 0.40}$ & $8.30_{\pm 0.36}$ & $8.29_{\pm 0.36}$ & $\underline{58.08}_{\pm 0.78}$ & N/A \\
                \cmidrule(lr){1-7}
                \multirow{7}{*}{\rotatebox{90}{Lipschitz}}
                & $L = 1$   & $63.14_{\pm 0.06}$ & $23.02_{\pm 0.03}$ & $-23.02_{\pm 0.03}$ & $47.45_{\pm 0.17}$ & $55.89_{\pm 0.09}$ \\
                % & $L = 2$   & $71.77_{\pm 0.18}$ & $19.09_{\pm 0.19}$ & $-19.09_{\pm 0.19}$ & $50.76_{\pm 0.12}$ & $60.69_{\pm 0.07}$ \\
                & $L = 4$   & $78.17_{\pm 0.15}$ & $13.46_{\pm 0.11}$ & $-13.45_{\pm 0.11}$ & $51.80_{\pm 0.24}$ & $\underline{62.12}_{\pm 0.27}$ \\
                % & $L = 8$   & $82.55_{\pm 0.09}$ & $7.91_{\pm 0.15}$ & $-7.90_{\pm 0.15}$ & $50.08_{\pm 0.61}$ & $58.23_{\pm 0.30}$ \\
                & $L = 16$  & $85.49_{\pm 0.15}$ & $3.21_{\pm 0.24}$ & $-3.18_{\pm 0.24}$ & $46.14_{\pm 0.28}$ & $47.20_{\pm 0.27}$ \\
                & $L = 32$  & $86.73_{\pm 0.22}$ & $1.60_{\pm 0.12}$ & $1.48_{\pm 0.18}$ & $41.93_{\pm 0.13}$ & $30.83_{\pm 0.15}$ \\
                % & $L = 64$  & $87.76_{\pm 0.04}$ & $4.26_{\pm 0.10}$ & $4.24_{\pm 0.07}$ & $37.15_{\pm 0.47}$ & $12.32_{\pm 0.26}$ \\
                & $L = 128$ & $88.55_{\pm 0.10}$ & $5.51_{\pm 0.05}$ & $5.51_{\pm 0.05}$ & $32.76_{\pm 0.45}$ & $0.77_{\pm 0.04}$ \\
                \cmidrule(lr){2-7}
                & \textbf{LiST ($\xi = 0.$)} & $86.59_{\pm 0.28}$ & $\mathbf{0.79}_{\pm 0.16}$ & $\underline{-0.15}_{\pm 0.30}$ & $43.29_{\pm 1.08}$ & $34.52_{\pm 1.76}$ \\
                & \textbf{LiST ($\xi = 3.$)} & $76.93_{\pm 0.20}$ & $\underline{1.17}_{\pm 0.21}$ & $\mathbf{0.11}_{\pm 0.13}$ & $53.04_{\pm 0.15}$ & $\mathbf{63.06}_{\pm 0.14}$ \\
                \bottomrule
            \end{tabular}
        \end{table*}
        \begin{figure}[t]
            \centering
            \includegraphics[width=\linewidth]{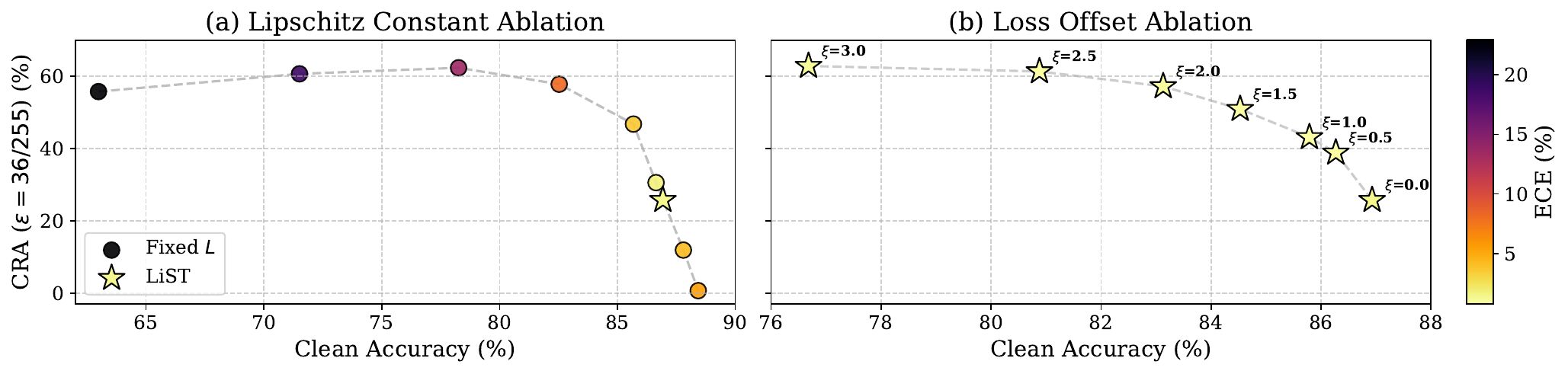}
            \caption{
            \textbf{Pareto Analysis on CIFAR-10.} (Left) LiST naturally converges to the calibrated point on the Pareto front. (Right) Varying $\xi$ yields a fully calibrated Pareto front using $100\%$ of training data.
            % \textbf{Pareto Analysis on CIFAR-10.} (Left) For a fixed offset $\xi = 0$, LiST converges to the point of the Pareto front that achieve calibration naturally. (Right) The offset enables to recreate a ``calibrated Pareto front'', and thus obtain highly robust and calibrated models that use 100\% of the training data.
            }
            \label{fig:pareto_analysis}
        \end{figure}

        The top block of Table~\ref{tab:low_data_and_datasets} reports LiST performance on CIFAR-10 as the training set size decreases from $N=50{,}000$ to $N=1{,}000$. Across all data scales, LiST consistently achieves low ECE, ranging from $1.27\%$ at full scale to $2.62\%$ in the extreme low-data regime. Importantly, this calibration is achieved without consuming any calibration data for post-hoc correction; $\mathcal{D}_{\mathrm{cal}}$ is used solely as a training signal. Notably, as $N$ decreases, the converged $L^*$ also decreases, reflecting the reduced task complexity. Yet the algorithm always identifies a non-trivial, calibrated operating point, confirming that LiST is robust to data scarcity.
        \begin{table*}[t]
            \centering
            \footnotesize
            \caption{
            \textbf{Ablations studies.} (Top) LiST on CIFAR-10 under varying training set sizes $N$. (Bottom) Phase~I vs. Phase~II comparison across datasets.
            % \textbf{LiST remains well calibrated across data regimes and benefits from Phase~II across datasets.} The top block reports LiST performance on CIFAR-10 as the number of training examples $N$ decreases. The bottom block compares Phase~I and Phase~II across datasets. It yields consistent improvements on accuracy and robustness while preserving strong calibration.
            }
            \label{tab:low_data_and_datasets}
            \tabcolsep=0.11cm
            \begin{tabular}{ccccccccc}
                \toprule
                \multirow{2}{*}{\textbf{Setting}} & \multirow{2}{*}{\textbf{Phase}} & \multirow{2}{*}{$\mathbf{L^*}$} & \multirow{2}{*}{\textbf{Clean Acc.} $(\uparrow)$} & \multirow{2}{*}{\textbf{ECE} $(\downarrow)$} & \multirow{2}{*}{\textbf{ESCE} $(\rightarrow 0)$} & \textbf{AA Acc.} & \textbf{CRA} \\
                 &  &  &  &  &  & $(\uparrow, \varepsilon = 0.5)$ & $(\uparrow, \varepsilon = 36/255)$ \\
                \midrule
                \multicolumn{8}{c}{\textit{Low-data regimes ablation}} \\
                \midrule
                $N=50000$ & \multirow{5}{*}{II} & $34.66$ & $86.89$ & $1.27$ & $1.25$ & $41.51$ & $25.78$ \\
                $N=25000$ &                     & $27.10$ & $83.70$ & $1.31$ & $1.20$ & $42.42$ & $30.11$ \\
                $N=10000$ &                     & $15.71$ & $77.36$ & $1.52$ & $0.69$ & $41.12$ & $37.58$ \\
                $N=5000$  &                     & $12.88$ & $71.74$ & $2.62$ & $2.59$ & $39.60$ & $37.42$ \\
                $N=1000$  &                     & $5.20$  & $52.53$ & $2.16$ & $0.99$ & $33.45$ & $30.06$ \\
                \midrule
                \multicolumn{8}{c}{\textit{Dataset and Phase II ablation}} \\
                \midrule
                \multirow{2}{*}{CIFAR10} & I & \multirow{2}{*}{$34.66$} & $86.87$ & $1.34$ & $\mathbf{1.18}$ & $41.49$ & $25.44$ \\
                % &  & \textcolor{green!75!black}{$(+0.02)$} & \textcolor{green!75!black}{$(-0.07)$} & \textcolor{red!90!black}{$(+0.07)$} & \textcolor{green!75!black}{$(+0.02)$} & \textcolor{green!75!black}{$(+0.34)$} \\
                 & II &  & $\mathbf{86.89}$ & $\mathbf{1.27}$ & $1.25$ & $\mathbf{41.51}$ & $\mathbf{25.78}$ \\
                \cmidrule(lr){1-8}
                \multirow{2}{*}{CIFAR-100} & I & \multirow{2}{*}{$24.96$} & $57.04$ & $\mathbf{1.52}$ & $-0.82$ & $21.20$ & $10.07$ \\
                % &  & \textcolor{green!75!black}{$(+0.29)$} & \textcolor{red!90!black}{$(+0.09)$} & \textcolor{green!75!black}{$(+0.02)$} & \textcolor{green!75!black}{$(+0.10)$} & \textcolor{green!75!black}{$(+0.32)$} \\
                 & II &  & $\mathbf{57.33}$ & $1.61$ & $\mathbf{-0.80}$ & $\mathbf{21.30}$ & $\mathbf{10.39}$ \\
                \cmidrule(lr){1-8}
                \multirow{2}{*}{Tiny-ImageNet} & I & \multirow{2}{*}{$14.72$} & $42.24$ & $2.42$ & $-2.19$ & $21.17$ & $12.34$ \\
                % &  & \textcolor{green!75!black}{$(+0.52)$} & \textcolor{green!75!black}{$(-0.31)$} & \textcolor{green!75!black}{$(+0.31)$} & \textcolor{green!75!black}{$(+0.33)$} & \textcolor{green!75!black}{$(+0.58)$} \\
                 & II &  & $\mathbf{42.76}$ & $\mathbf{2.11}$ & $\mathbf{-1.88}$ & $\mathbf{21.50}$ & $\mathbf{12.92}$ \\
                \bottomrule
            \end{tabular}
        \end{table*}

        The bottom block of Table~\ref{tab:low_data_and_datasets} reports LiST performance before and after Phase~II on CIFAR-10, CIFAR-100, and Tiny-ImageNet. Across all three datasets, LiST reaches a well-calibrated state at the end of Phase~I, with ECE below $2.42\%$ in all cases. Phase~II consistently improves clean accuracy and robustness across all datasets, with gains up to $+0.52\%$ in accuracy and $+0.58\%$ in CRA on Tiny-ImageNet, while maintaining calibration in a comparable range. These improvements are consistent with the theoretical guarantee of Section~\ref{subsec:stopping_criterion_and_data_reintegration}: re-integrating $\mathcal{D}_{\mathrm{cal}}$ into training strictly tightens the generalization bound without altering model capacity.
        
        \subsection{Limitations}
        \label{subsec:limitations}
    
        While LiST significantly simplifies the training pipeline by eliminating the need for a grid search over the Lipschitz constant, several limitations remain. First, the existence of $L^*$ is guaranteed by~Proposition~\ref{prop:existence_L_star}, but its uniqueness is not, and while we conjecture that a unique non-trivial L* exists under mild regularity conditions, this remains an open theoretical question. Furthermore, $L^*$ depends on the full training scheme (architecture, data augmentation, optimizer) and cannot be transferred across settings; however, this is no worse than the existing practice of grid-searching over $L$, which LiST precisely replaces. Finally, the calibration criterion selects a principled operating point on the accuracy-robustness Pareto front, which does not necessarily maximize certified robustness. In such cases, the loss offset $\xi$ already present in prior Lipschitz training works can be used to navigate the trade-off, so that LiST effectively reduces the hyperparameter burden from two ($L$ and $\xi$) to one ($\xi$ alone).

    \section{Conclusion}
    \label{sec:conclusion}

        Accuracy, robustness, and calibration are three pillars of trustworthy machine learning, yet they have long been studied in isolation. In this work, we have shown that these objectives are more deeply connected than previously recognized. By formalizing the structural duality between the Lipschitz constant and Temperature Scaling, we established that calibration acts as a principled, unsupervised criterion to identify a natural operating point on the accuracy-robustness Pareto front. Building on this insight, LiST automates the search for this operating point, yielding models that are simultaneously certifiably robust and calibrated out-of-the-box, without any manual tuning of the Lipschitz constant. Furthermore, by varying the loss offset $\xi$, LiST enables the construction of a fully calibrated Pareto front, decoupling calibration from the accuracy-robustness trade-off. Finally, Phase II demonstrates that once $L^*$ is identified, the calibration data can be safely reintegrated into training, improving sample efficiency at no cost to calibration.
        
        Several directions remain open. The theoretical uniqueness of $L^*$ is an interesting open question, and extending LiST to other architectures, modalities, or tasks such as out-of-distribution detection would be a natural next step. More broadly, we believe that using calibration as a training signal, rather than a post-hoc correction, opens a promising avenue toward more principled and trustworthy deep learning.

    %\ifarxiv
    %\if false %\anonymous
    
    \section*{Acknowledgments}
    
    Our work has benefited from the AI Cluster ANITI and the research program DEEL.\footnote{\url{https://www.deel.ai/}} ANITI is funded by the France 2030 program under the Grant agreement n°ANR-23-IACL-0002. DEEL is an integrative program of the AI Cluster ANITI, designed and operated jointly with IRT Saint Exupéry, with the financial support from its industrial and academic partners and the France 2030 program under the Grant agreement n°ANR-10-AIRT-01.
    
    \noindent This work was granted access to the HPC resources of IDRIS under the allocation 2025-AD011017289 made by GENCI. 
    %\fi

    \bibliographystyle{plainnat}  %ieeenat_fullname
    \bibliography{bibliography}

\begin{thebibliography}{53}
\providecommand{\natexlab}[1]{#1}
\providecommand{\url}[1]{\texttt{#1}}
\expandafter\ifx\csname urlstyle\endcsname\relax
  \providecommand{\doi}[1]{doi: #1}\else
  \providecommand{\doi}{doi: \begingroup \urlstyle{rm}\Url}\fi

\bibitem[Angelopoulos and Bates(2023)]{angelopoulos_conformal_2023}
Anastasios~N. Angelopoulos and Stephen Bates.
\newblock Conformal {Prediction}: {A} {Gentle} {Introduction}.
\newblock \emph{Foundations and Trends in Machine Learning}, 16\penalty0
  (4):\penalty0 494--591, March 2023.
\newblock ISSN 1935-8237.
\newblock \doi{10.1561/2200000101}.
\newblock URL \url{https://doi.org/10.1561/2200000101}.

\bibitem[Anil et~al.(2019)Anil, Lucas, and Grosse]{anil_sorting_2019}
Cem Anil, James Lucas, and Roger Grosse.
\newblock Sorting {Out} {Lipschitz} {Function} {Approximation}.
\newblock In \emph{Proceedings of the 36th {International} {Conference} on
  {Machine} {Learning}}, pages 291--301. PMLR, May 2019.
\newblock URL \url{https://proceedings.mlr.press/v97/anil19a.html}.

\bibitem[Araujo et~al.(2023)Araujo, Havens, Delattre, Allauzen, and
  Hu]{araujo2023unified}
Alexandre Araujo, Aaron~J Havens, Blaise Delattre, Alexandre Allauzen, and Bin
  Hu.
\newblock A unified algebraic perspective on lipschitz neural networks.
\newblock In \emph{The Eleventh International Conference on Learning
  Representations}, 2023.

\bibitem[Bartlett et~al.(2017)Bartlett, Foster, and
  Telgarsky]{bartlett_spectrally-normalized_2017}
Peter~L Bartlett, Dylan~J Foster, and Matus~J Telgarsky.
\newblock Spectrally-normalized margin bounds for neural networks.
\newblock In \emph{Advances in {Neural} {Information} {Processing} {Systems}},
  volume~30. Curran Associates, Inc., 2017.
\newblock URL
  \url{https://proceedings.neurips.cc/paper_files/paper/2017/hash/b22b257ad0519d4500539da3c8bcf4dd-Abstract.html}.

\bibitem[Boissin et~al.(2025)Boissin, Mamalet, Fel, Picard, Massena, and
  Serrurier]{boissin2025adaptive}
Thibaut Boissin, Franck Mamalet, Thomas Fel, Agustin~Martin Picard, Thomas
  Massena, and Mathieu Serrurier.
\newblock An adaptive orthogonal convolution scheme for efficient and flexible
  cnn architectures.
\newblock In \emph{International Conference on Machine Learning}, pages
  4757--4790. PMLR, 2025.

\bibitem[Boyd and Vandenberghe(2023)]{boyd_convex_2023}
Stephen~P. Boyd and Lieven Vandenberghe.
\newblock \emph{Convex optimization}.
\newblock Cambridge University Press, Cambridge New York Melbourne New Delhi
  Singapore, version 29 edition, 2023.
\newblock ISBN 978-0-521-83378-3.

\bibitem[Brier(1950)]{brier_verification_1950}
Glenn~W. Brier.
\newblock {VERIFICATION} {OF} {FORECASTS} {EXPRESSED} {IN} {TERMS} {OF}
  {PROBABILITY}.
\newblock \emph{Monthly Weather Review}, 78\penalty0 (1):\penalty0 1--3,
  January 1950.
\newblock ISSN 1520-0493, 0027-0644.
\newblock \doi{10.1175/1520-0493(1950)078<0001:VOFEIT>2.0.CO;2}.
\newblock URL
  \url{https://journals.ametsoc.org/view/journals/mwre/78/1/1520-0493_1950_078_0001_vofeit_2_0_co_2.xml}.

\bibitem[Béthune et~al.(2022)Béthune, Boissin, Serrurier, Mamalet, Friedrich,
  and Gonzalez~Sanz]{bethune_pay_2022}
Louis Béthune, Thibaut Boissin, Mathieu Serrurier, Franck Mamalet, Corentin
  Friedrich, and Alberto Gonzalez~Sanz.
\newblock Pay attention to your loss : understanding misconceptions about
  {Lipschitz} neural networks.
\newblock \emph{Advances in Neural Information Processing Systems},
  35:\penalty0 20077--20091, December 2022.
\newblock URL
  \url{https://papers.nips.cc/paper_files/paper/2022/hash/7eb3d8ae592966543170a65e6b698828-Abstract-Conference.html}.

\bibitem[Croce and Hein(2020)]{croce_reliable_2020}
Francesco Croce and Matthias Hein.
\newblock Reliable evaluation of adversarial robustness with an ensemble of
  diverse parameter-free attacks.
\newblock In \emph{Proceedings of the 37th {International} {Conference} on
  {Machine} {Learning}}, pages 2206--2216. PMLR, November 2020.
\newblock URL \url{https://proceedings.mlr.press/v119/croce20b.html}.

\bibitem[Defazio et~al.(2024)Defazio, Yang, Mehta, Mishchenko, Khaled, and
  Cutkosky]{defazio_road_2024}
Aaron Defazio, Xingyu Yang, Harsh Mehta, Konstantin Mishchenko, Ahmed Khaled,
  and Ashok Cutkosky.
\newblock The {Road} {Less} {Scheduled}.
\newblock In A.~Globerson, L.~Mackey, D.~Belgrave, A.~Fan, U.~Paquet,
  J.~Tomczak, and C.~Zhang, editors, \emph{Advances in {Neural} {Information}
  {Processing} {Systems}}, volume~37, pages 9974--10007. Curran Associates,
  Inc., 2024.
\newblock \doi{10.52202/079017-0320}.
\newblock URL
  \url{https://proceedings.neurips.cc/paper_files/paper/2024/file/136b9a13861308c8948cd308ccd02658-Paper-Conference.pdf}.

\bibitem[Dwork et~al.(2012)Dwork, Hardt, Pitassi, Reingold, and
  Zemel]{dwork_fairness_2012}
Cynthia Dwork, Moritz Hardt, Toniann Pitassi, Omer Reingold, and Richard Zemel.
\newblock Fairness through awareness.
\newblock In \emph{Proceedings of the 3rd {Innovations} in {Theoretical}
  {Computer} {Science} {Conference}}, {ITCS} '12, pages 214--226, New York, NY,
  USA, January 2012. Association for Computing Machinery.
\newblock ISBN 978-1-4503-1115-1.
\newblock \doi{10.1145/2090236.2090255}.
\newblock URL \url{https://dl.acm.org/doi/10.1145/2090236.2090255}.

\bibitem[Galil and El-Yaniv(2021)]{galil_disrupting_2021}
Ido Galil and Ran El-Yaniv.
\newblock Disrupting {Deep} {Uncertainty} {Estimation} {Without} {Harming}
  {Accuracy}.
\newblock In \emph{Advances in {Neural} {Information} {Processing} {Systems}},
  volume~34, pages 21285--21296. Curran Associates, Inc., 2021.
\newblock URL
  \url{https://proceedings.neurips.cc/paper/2021/hash/b1b20d09041289e6c3fbb81850c5da54-Abstract.html}.

\bibitem[Golowich et~al.(2018)Golowich, Rakhlin, and Shamir]{golowich2018size}
Noah Golowich, Alexander Rakhlin, and Ohad Shamir.
\newblock Size-independent sample complexity of neural networks.
\newblock In \emph{Conference on learning theory}, pages 297--299. PMLR, 2018.

\bibitem[Goodfellow et~al.(2015)Goodfellow, Shlens, and
  Szegedy]{goodfellow2015explaining}
Ian~J. Goodfellow, Jonathon Shlens, and Christian Szegedy.
\newblock Explaining and harnessing adversarial examples.
\newblock In \emph{International Conference on Learning Representations
  (ICLR)}, 2015.
\newblock URL \url{https://arxiv.org/abs/1412.6572}.

\bibitem[Grabinski et~al.(2022)Grabinski, Gavrikov, Keuper, and
  Keuper]{grabinski_robust_2022}
Julia Grabinski, Paul Gavrikov, Janis Keuper, and Margret Keuper.
\newblock Robust {Models} are less {Over}-{Confident}, December 2022.
\newblock URL \url{http://arxiv.org/abs/2210.05938}.
\newblock arXiv:2210.05938 [cs].

\bibitem[Grosse et~al.(2019)Grosse, Pfaff, Smith, and
  Backes]{grosse_limitations_2019}
Kathrin Grosse, David Pfaff, Michael~Thomas Smith, and Michael Backes.
\newblock The {Limitations} of {Model} {Uncertainty} in {Adversarial}
  {Settings}, November 2019.
\newblock URL \url{http://arxiv.org/abs/1812.02606}.
\newblock arXiv:1812.02606 [cs].

\bibitem[Guo et~al.(2017)Guo, Pleiss, Sun, and
  Weinberger]{guo_calibration_2017}
Chuan Guo, Geoff Pleiss, Yu~Sun, and Kilian~Q. Weinberger.
\newblock On {Calibration} of {Modern} {Neural} {Networks}.
\newblock In \emph{Proceedings of the 34th {International} {Conference} on
  {Machine} {Learning}}, pages 1321--1330. PMLR, July 2017.
\newblock URL \url{https://proceedings.mlr.press/v70/guo17a.html}.
\newblock ISSN: 2640-3498.

\bibitem[Hardt et~al.(2016)Hardt, Price, and Srebro]{hardt_equality_2016}
Moritz Hardt, Eric Price, and Nati Srebro.
\newblock Equality of {Opportunity} in {Supervised} {Learning}.
\newblock In \emph{Advances in {Neural} {Information} {Processing} {Systems}},
  volume~29. Curran Associates, Inc., 2016.
\newblock URL
  \url{https://proceedings.neurips.cc/paper_files/paper/2016/hash/6a9659feb1216f14f7384ba499518b38-Abstract.html}.

\bibitem[Krishnan and Tickoo(2020)]{krishnan_improving_2020}
Ranganath Krishnan and Omesh Tickoo.
\newblock Improving model calibration with accuracy versus uncertainty
  optimization.
\newblock In \emph{Advances in {Neural} {Information} {Processing} {Systems}},
  volume~33, pages 18237--18248. Curran Associates, Inc., 2020.
\newblock URL
  \url{https://proceedings.neurips.cc/paper_files/paper/2020/hash/d3d9446802a44259755d38e6d163e820-Abstract.html}.

\bibitem[Kull et~al.(2019)Kull, Perello~Nieto, Kängsepp, Silva~Filho, Song,
  and Flach]{kull_beyond_2019}
Meelis Kull, Miquel Perello~Nieto, Markus Kängsepp, Telmo Silva~Filho, Hao
  Song, and Peter Flach.
\newblock Beyond temperature scaling: {Obtaining} well-calibrated multi-class
  probabilities with {Dirichlet} calibration.
\newblock In \emph{Advances in {Neural} {Information} {Processing} {Systems}},
  volume~32. Curran Associates, Inc., 2019.
\newblock URL
  \url{https://proceedings.neurips.cc/paper/2019/hash/8ca01ea920679a0fe3728441494041b9-Abstract.html}.

\bibitem[Kumar et~al.(2018)Kumar, Sarawagi, and Jain]{kumar_trainable_2018}
Aviral Kumar, Sunita Sarawagi, and Ujjwal Jain.
\newblock Trainable {Calibration} {Measures} for {Neural} {Networks} from
  {Kernel} {Mean} {Embeddings}.
\newblock In \emph{Proceedings of the 35th {International} {Conference} on
  {Machine} {Learning}}, pages 2805--2814. PMLR, July 2018.
\newblock URL \url{https://proceedings.mlr.press/v80/kumar18a.html}.

\bibitem[Lai et~al.(2025)Lai, Huang, Kung, and Chen]{lai2025enhancing}
Bo-Han Lai, Pin-Han Huang, Bo-Han Kung, and Shang-Tse Chen.
\newblock Enhancing certified robustness via block reflector orthogonal layers
  and logit annealing loss.
\newblock In \emph{International Conference on Machine Learning (ICML)}, 2025.
\newblock Spotlight.

\bibitem[Ledda et~al.(2025)Ledda, Scodeller, Angioni, Piras, Cinà, Fumera,
  Biggio, and Roli]{ledda_robustness_2025}
Emanuele Ledda, Giovanni Scodeller, Daniele Angioni, Giorgio Piras,
  Antonio~Emanuele Cinà, Giorgio Fumera, Battista Biggio, and Fabio Roli.
\newblock On the {Robustness} of {Adversarial} {Training} {Against}
  {Uncertainty} {Attacks}, May 2025.
\newblock URL \url{http://arxiv.org/abs/2410.21952}.
\newblock arXiv:2410.21952 [cs].

\bibitem[Li et~al.(2019)Li, Haque, Anil, Lucas, Grosse, and
  Jacobsen]{li_preventing_2019}
Qiyang Li, Saminul Haque, Cem Anil, James Lucas, Roger~B Grosse, and
  Joern-Henrik Jacobsen.
\newblock Preventing {Gradient} {Attenuation} in {Lipschitz} {Constrained}
  {Convolutional} {Networks}.
\newblock In \emph{Advances in {Neural} {Information} {Processing} {Systems}},
  volume~32. Curran Associates, Inc., 2019.
\newblock URL
  \url{https://proceedings.neurips.cc/paper_files/paper/2019/hash/1ce3e6e3f452828e23a0c94572bef9d9-Abstract.html}.

\bibitem[Liu et~al.(2022)Liu, Ben~Ayed, Galdran, and Dolz]{liu2022devil}
Bingyuan Liu, Ismail Ben~Ayed, Adrian Galdran, and Jose Dolz.
\newblock The devil is in the margin: Margin-based label smoothing for network
  calibration.
\newblock In \emph{Proceedings of the IEEE/CVF Conference on Computer Vision
  and Pattern Recognition}, pages 80--88, 2022.

\bibitem[Lundberg and Lee(2017)]{lundberg_unified_2017}
Scott~M Lundberg and Su-In Lee.
\newblock A {Unified} {Approach} to {Interpreting} {Model} {Predictions}.
\newblock In \emph{Advances in {Neural} {Information} {Processing} {Systems}},
  volume~30. Curran Associates, Inc., 2017.
\newblock URL
  \url{https://proceedings.neurips.cc/paper_files/paper/2017/hash/8a20a8621978632d76c43dfd28b67767-Abstract.html}.

\bibitem[Madry et~al.(2018)Madry, Makelov, Schmidt, Tsipras, and
  Vladu]{madry_towards_2018}
Aleksander Madry, Aleksandar Makelov, Ludwig Schmidt, Dimitris Tsipras, and
  Adrian Vladu.
\newblock Towards {Deep} {Learning} {Models} {Resistant} to {Adversarial}
  {Attacks}.
\newblock In \emph{International Conference on Learning Representations},
  February 2018.
\newblock URL \url{https://openreview.net/forum?id=rJzIBfZAb}.

\bibitem[Meunier et~al.(2022)Meunier, Delattre, Araujo, and
  Allauzen]{meunier_dynamical_2022}
Laurent Meunier, Blaise~J. Delattre, Alexandre Araujo, and Alexandre Allauzen.
\newblock A {Dynamical} {System} {Perspective} for {Lipschitz} {Neural}
  {Networks}.
\newblock In \emph{Proceedings of the 39th {International} {Conference} on
  {Machine} {Learning}}, pages 15484--15500. PMLR, June 2022.
\newblock URL \url{https://proceedings.mlr.press/v162/meunier22a.html}.

\bibitem[Miyato et~al.(2018)Miyato, Kataoka, Koyama, and
  Yoshida]{miyato_spectral_2018}
Takeru Miyato, Toshiki Kataoka, Masanori Koyama, and Yuichi Yoshida.
\newblock Spectral {Normalization} for {Generative} {Adversarial} {Networks}.
\newblock In \emph{International Conference on Learning Representations},
  February 2018.
\newblock URL
  \url{https://openreview.net/forum?id=B1QRgziT-&source=post_page---------------------------}.

\bibitem[Mukhoti et~al.(2020)Mukhoti, Kulharia, Sanyal, Golodetz, Torr, and
  Dokania]{mukhoti_calibrating_2020}
Jishnu Mukhoti, Viveka Kulharia, Amartya Sanyal, Stuart Golodetz, Philip Torr,
  and Puneet Dokania.
\newblock Calibrating {Deep} {Neural} {Networks} using {Focal} {Loss}.
\newblock In \emph{Advances in {Neural} {Information} {Processing} {Systems}},
  volume~33, pages 15288--15299. Curran Associates, Inc., 2020.
\newblock URL
  \url{https://proceedings.neurips.cc/paper_files/paper/2020/hash/aeb7b30ef1d024a76f21a1d40e30c302-Abstract.html}.

\bibitem[Murphy and Winkler(1977)]{murphy_reliability_1977}
Allan~H. Murphy and Robert~L. Winkler.
\newblock Reliability of {Subjective} {Probability} {Forecasts} of
  {Precipitation} and {Temperature}.
\newblock \emph{Journal of the Royal Statistical Society. Series C (Applied
  Statistics)}, 26\penalty0 (1):\penalty0 41--47, 1977.
\newblock ISSN 0035-9254.
\newblock \doi{10.2307/2346866}.
\newblock URL \url{https://www.jstor.org/stable/2346866}.
\newblock Publisher: [Royal Statistical Society, Oxford University Press].

\bibitem[Müller et~al.(2019)Müller, Kornblith, and Hinton]{muller_when_2019}
Rafael Müller, Simon Kornblith, and Geoffrey~E. Hinton.
\newblock When does label smoothing help?
\newblock \emph{Advances in Neural Information Processing Systems}, 32, 2019.
\newblock URL
  \url{https://proceedings.neurips.cc/paper_files/paper/2019/hash/f1748d6b0fd9d439f71450117eba2725-Abstract.html?ref=gojiberries.io}.

\bibitem[Naeini et~al.(2015)Naeini, Cooper, and
  Hauskrecht]{naeini_obtaining_2015}
Mahdi~Pakdaman Naeini, Gregory Cooper, and Milos Hauskrecht.
\newblock Obtaining {Well} {Calibrated} {Probabilities} {Using} {Bayesian}
  {Binning}.
\newblock \emph{Proceedings of the AAAI Conference on Artificial Intelligence},
  29\penalty0 (1), February 2015.
\newblock ISSN 2374-3468.
\newblock \doi{10.1609/aaai.v29i1.9602}.
\newblock URL \url{https://ojs.aaai.org/index.php/AAAI/article/view/9602}.
\newblock Number: 1.

\bibitem[Neyshabur et~al.(2018)Neyshabur, Bhojanapalli, and
  Srebro]{neyshabur2018pac}
Behnam Neyshabur, Srinadh Bhojanapalli, and Nathan Srebro.
\newblock A pac-bayesian approach to spectrally-normalized margin bounds for
  neural networks.
\newblock In \emph{International Conference on Learning Representations}, 2018.

\bibitem[Obadinma et~al.(2024)Obadinma, Zhu, and
  Guo]{obadinma_calibration_2024}
Stephen Obadinma, Xiaodan Zhu, and Hongyu Guo.
\newblock Calibration {Attacks}: {A} {Comprehensive} {Study} of {Adversarial}
  {Attacks} on {Model} {Confidence}.
\newblock \emph{Transactions on Machine Learning Research}, May 2024.
\newblock ISSN 2835-8856.
\newblock URL \url{https://openreview.net/forum?id=TXzz9xwdpv}.

\bibitem[Prach and Lampert(2022)]{prach_almost-orthogonal_2023}
Bernd Prach and Christoph~H Lampert.
\newblock Almost-orthogonal layers for efficient general-purpose lipschitz
  networks.
\newblock In \emph{European Conference on Computer Vision}, pages 350--365.
  Springer, 2022.

\bibitem[Prach and Lampert(2025)]{prach2025intriguing}
Bernd Prach and Christoph~H Lampert.
\newblock Intriguing properties of robust classification.
\newblock In \emph{Proceedings of the Computer Vision and Pattern Recognition
  Conference}, pages 660--669, 2025.

\bibitem[Qin et~al.(2021)Qin, Wang, Beutel, and Chi]{qin_improving_2021}
Yao Qin, Xuezhi Wang, Alex Beutel, and Ed~Chi.
\newblock Improving {Calibration} through the {Relationship} with {Adversarial}
  {Robustness}.
\newblock In \emph{Advances in {Neural} {Information} {Processing} {Systems}},
  volume~34, pages 14358--14369. Curran Associates, Inc., 2021.
\newblock URL
  \url{https://proceedings.neurips.cc/paper_files/paper/2021/hash/78421a2e0e1168e5cd1b7a8d23773ce6-Abstract.html}.

\bibitem[Ribeiro et~al.(2016)Ribeiro, Singh, and Guestrin]{ribeiro_why_2016}
Marco~Tulio Ribeiro, Sameer Singh, and Carlos Guestrin.
\newblock "{Why} {Should} {I} {Trust} {You}?": {Explaining} the {Predictions}
  of {Any} {Classifier}.
\newblock In \emph{Proceedings of the 22nd {ACM} {SIGKDD} {International}
  {Conference} on {Knowledge} {Discovery} and {Data} {Mining}}, pages
  1135--1144, San Francisco California USA, August 2016. ACM.
\newblock ISBN 978-1-4503-4232-2.
\newblock \doi{10.1145/2939672.2939778}.
\newblock URL \url{https://dl.acm.org/doi/10.1145/2939672.2939778}.

\bibitem[Serrurier et~al.(2021)Serrurier, Mamalet, Gonzalez-Sanz, Boissin,
  Loubes, and del Barrio]{serrurier_achieving_2021}
Mathieu Serrurier, Franck Mamalet, Alberto Gonzalez-Sanz, Thibaut Boissin,
  Jean-Michel Loubes, and Eustasio del Barrio.
\newblock Achieving {Robustness} in {Classification} {Using} {Optimal}
  {Transport} {With} {Hinge} {Regularization}.
\newblock In \emph{Proceedings of the Computer Vision and Pattern Recognition
  Conference}, pages 505--514, 2021.
\newblock URL
  \url{https://openaccess.thecvf.com/content/CVPR2021/html/Serrurier_Achieving_Robustness_in_Classification_Using_Optimal_Transport_With_Hinge_Regularization_CVPR_2021_paper.html}.

\bibitem[Singla and Feizi(2021)]{singla_skew_2021}
Sahil Singla and Soheil Feizi.
\newblock Skew {Orthogonal} {Convolutions}.
\newblock In \emph{Proceedings of the 38th {International} {Conference} on
  {Machine} {Learning}}, pages 9756--9766. PMLR, July 2021.
\newblock URL \url{https://proceedings.mlr.press/v139/singla21a.html}.

\bibitem[Szegedy et~al.(2014)Szegedy, Zaremba, Sutskever, Bruna, Erhan,
  Goodfellow, and Fergus]{szegedy2014intriguing}
Christian Szegedy, Wojciech Zaremba, Ilya Sutskever, Joan Bruna, Dumitru Erhan,
  Ian Goodfellow, and Rob Fergus.
\newblock Intriguing properties of neural networks.
\newblock In \emph{International Conference on Learning Representations
  (ICLR)}, 2014.
\newblock URL \url{https://arxiv.org/abs/1312.6199}.

\bibitem[Tao et~al.(2023)Tao, Dong, and Xu]{tao_dual_2023}
Linwei Tao, Minjing Dong, and Chang Xu.
\newblock Dual {Focal} {Loss} for {Calibration}.
\newblock In \emph{Proceedings of the 40th {International} {Conference} on
  {Machine} {Learning}}, pages 33833--33849. PMLR, July 2023.
\newblock URL \url{https://proceedings.mlr.press/v202/tao23a.html}.

\bibitem[Trockman and Kolter(2021)]{trockmanorthogonalizing}
Asher Trockman and J~Zico Kolter.
\newblock Orthogonalizing convolutional layers with the cayley transform.
\newblock In \emph{International Conference on Learning Representations}, 2021.

\bibitem[Tsuzuku et~al.(2018)Tsuzuku, Sato, and
  Sugiyama]{tsuzuku_lipschitz-margin_2018}
Yusuke Tsuzuku, Issei Sato, and Masashi Sugiyama.
\newblock Lipschitz-{Margin} {Training}: {Scalable} {Certification} of
  {Perturbation} {Invariance} for {Deep} {Neural} {Networks}.
\newblock In \emph{Advances in {Neural} {Information} {Processing} {Systems}},
  volume~31. Curran Associates, Inc., 2018.
\newblock URL
  \url{https://proceedings.neurips.cc/paper/2018/hash/485843481a7edacbfce101ecb1e4d2a8-Abstract.html}.

\bibitem[Vasilev and D'yakonov(2023)]{vasilev_calibration_2023}
Ruslan Vasilev and Alexander D'yakonov.
\newblock Calibration of {Neural} {Networks}, March 2023.
\newblock URL \url{http://arxiv.org/abs/2303.10761}.
\newblock arXiv:2303.10761.

\bibitem[Verhaeghe et~al.(2023)Verhaeghe, De~Corte, Sauer, Hendriks, Thijssens,
  Ongenae, Elbers, De~Waele, and Van~Hoecke]{verhaeghe_generalizable_2023}
Jarne Verhaeghe, Thomas De~Corte, Christopher~M. Sauer, Tom Hendriks, Olivier
  W.~M. Thijssens, Femke Ongenae, Paul Elbers, Jan De~Waele, and Sofie
  Van~Hoecke.
\newblock Generalizable calibrated machine learning models for real-time atrial
  fibrillation risk prediction in {ICU} patients.
\newblock \emph{International Journal of Medical Informatics}, 175:\penalty0
  105086, July 2023.
\newblock ISSN 1386-5056.
\newblock \doi{10.1016/j.ijmedinf.2023.105086}.
\newblock URL
  \url{https://www.sciencedirect.com/science/article/pii/S1386505623001041}.

\bibitem[Virmaux and Scaman(2018)]{virmaux_lipschitz_2018}
Aladin Virmaux and Kevin Scaman.
\newblock Lipschitz regularity of deep neural networks: analysis and efficient
  estimation.
\newblock In \emph{Advances in {Neural} {Information} {Processing} {Systems}},
  volume~31. Curran Associates, Inc., 2018.
\newblock URL
  \url{https://proceedings.neurips.cc/paper_files/paper/2018/hash/d54e99a6c03704e95e6965532dec148b-Abstract.html}.

\bibitem[Vovk et~al.(2022)Vovk, Gammerman, and Shafer]{vovk_algorithmic_2022}
Vladimir Vovk, Alexander Gammerman, and Glenn Shafer.
\newblock \emph{Algorithmic {Learning} in a {Random} {World}}.
\newblock Springer International Publishing, Cham, 2022.
\newblock ISBN 978-3-031-06648-1 978-3-031-06649-8.
\newblock \doi{10.1007/978-3-031-06649-8}.
\newblock URL \url{https://link.springer.com/10.1007/978-3-031-06649-8}.

\bibitem[Xu et~al.(2022)Xu, Li, and Li]{xu_lot_2022}
Xiaojun Xu, Linyi Li, and Bo~Li.
\newblock {LOT}: {Layer}-wise {Orthogonal} {Training} on {Improving} l2
  {Certified} {Robustness}.
\newblock \emph{Advances in Neural Information Processing Systems},
  35:\penalty0 18904--18915, December 2022.
\newblock URL
  \url{https://proceedings.neurips.cc/paper_files/paper/2022/hash/77d52754ff6b2de5a5d96ee921b6b3cd-Abstract-Conference.html}.

\bibitem[Ye et~al.(2023)Ye, Ma, Cao, and Tang]{ye_mitigating_2023}
Wenqian Ye, Yunsheng Ma, Xu~Cao, and Kun Tang.
\newblock Mitigating {Transformer} {Overconfidence} via {Lipschitz}
  {Regularization}.
\newblock In \emph{Proceedings of the {Thirty}-{Ninth} {Conference} on
  {Uncertainty} in {Artificial} {Intelligence}}, pages 2422--2432. PMLR, July
  2023.
\newblock URL \url{https://proceedings.mlr.press/v216/ye23a.html}.

\bibitem[Zadrozny and Elkan(2001)]{zadrozny_obtaining_2001}
Bianca Zadrozny and Charles Elkan.
\newblock Obtaining calibrated probability estimates from decision trees and
  naive {Bayesian} classiﬁers.
\newblock \emph{International Conference on Machine Learning (ICML)}, 2001.

\bibitem[Zadrozny and Elkan(2002)]{zadrozny_transforming_2002}
Bianca Zadrozny and Charles Elkan.
\newblock Transforming {Classifier} {Scores} into {Accurate} {Multiclass}
  {Probability} {Estimates}.
\newblock \emph{Proceedings of the ACM SIGKDD International Conference on
  Knowledge Discovery and Data Mining}, August 2002.
\newblock \doi{10.1145/775047.775151}.

\end{thebibliography}

    %%%%%%%%%%%%%%%%%%%%%%%%%%%%%%%%%%%%%%%%%%%%%%%%%%%%%%%%%%%%%%%%%%%%%%%%%%%%%%%
    %%%%%%%%%%%%%%%%%%%%%%%%%%%%%%%%%%%%%%%%%%%%%%%%%%%%%%%%%%%%%%%%%%%%%%%%%%%%%%%
    % APPENDIX
    %%%%%%%%%%%%%%%%%%%%%%%%%%%%%%%%%%%%%%%%%%%%%%%%%%%%%%%%%%%%%%%%%%%%%%%%%%%%%%%
    %%%%%%%%%%%%%%%%%%%%%%%%%%%%%%%%%%%%%%%%%%%%%%%%%%%%%%%%%%%%%%%%%%%%%%%%%%%%%%%
    \newpage
    \appendix
    \onecolumn

    \section{LiST Algorithm}
    \label{app:algorithm}

        \Cref{alg:list} provides the complete pseudocode for LiST. The procedure consists of two phases. In Phase~I, the Lipschitz constraint $L_e$ is updated at each epoch via the feedback rule $L_{e+1} \leftarrow L_e / T^*_e$, where $T^*_e$ is the optimal calibration temperature measured on $\mathcal{D}_\text{cal}$. This establishes a negative feedback loop that drives the network toward the calibrated equilibrium $T^* = 1$ (see Section~\ref{subsec:convergence_dynamics} for a detailed discussion of the convergence dynamics). Training is halted when $L$ has sufficiently converged (see~\Cref{subsec:stopping_criterion_and_data_reintegration}). In Phase~II, the constraint is frozen at $L^*$ and the network is fine-tuned on the combined dataset $\mathcal{D}_\text{total} = \mathcal{D}_\text{train} \cup \mathcal{D}_\text{cal}$, whose justification is discussed in~\Cref{subsec:stopping_criterion_and_data_reintegration}.
        \begin{algorithm}
            \caption{LiST: Lipschitz Scaling Training}
            \label{alg:list}
            \begin{algorithmic}
                \STATE {\bfseries Input:} Training set $\mathcal{D}_{train}$, Calibration set $\mathcal{D}_{cal}$, Initial constant $L_{0}$, Learning rate $\eta$. 
                \STATE {\bfseries Initialize:} $L_{0}$-Lipschitz network $f^{(L_0)}_{\theta}$. 
                \STATE 
                \STATE \textit{\color{blue} \textbf{// Phase I: Dynamic Search}} 
                \STATE $e \leftarrow 0$ 
                \STATE $L_{e}\leftarrow L_{0}$ 
                \WHILE{True} 
                    \STATE \textit{\color{gray} // 1. Optimize weights under current constraint $L_{e}$} 
                    \FOR{batch $(x, y)$ {\bfseries in} $\mathcal{D}_{train}$}
                        \STATE $\mathcal{L}\leftarrow \mathcal{L}(f^{(L_e)}_{\theta}(x), y)$
                        \STATE $\theta \leftarrow \theta - \eta \text{Optimizer}(\nabla_{\theta}\mathcal{L})$
                    \ENDFOR 
                    \STATE 
                    \STATE \textit{\color{gray} // 2. Update Constraint via Feedback Loop} 
                    \STATE Compute logits $z = f^{(L_e)}_{\theta}(x_{cal})$ on $\mathcal{D}_{cal}$ 
                    \STATE $T^{*}\leftarrow \arg\min_{T}\mathcal{L}_{CE}(\sigma(z/T), y_{cal})$ 
                    \STATE $L_{e+1}\leftarrow L_{e}/ T^{*}$ 
                    \STATE Project $f_{\theta}$ onto constraint $L_{e+1}$ \COMMENT{see~\Cref{app:implementation}}
                    \STATE $e \leftarrow e + 1$ 
                    \STATE 
                    \STATE \textit{\color{gray} // 3. Stopping Criteria on $L$ evolution} 
                    \IF{no more evolution on $L$}
                        \STATE break \COMMENT{see~\Cref{subsec:stopping_criterion_and_data_reintegration}} 
                    \ENDIF
                \ENDWHILE 
                \STATE 
                \STATE \textit{\color{blue} \textbf{// Phase~II: Data Re-integration}} 
                \STATE $L^{*}\leftarrow L_{e}$ 
                \STATE $\mathcal{D}_{total}\leftarrow \mathcal{D}_{train}\cup \mathcal{D}_{cal}$
                \FOR{epoch $e'=1$ {\bfseries to} $E$} 
                    \FOR{batch $(x, y)$ {\bfseries in} $\mathcal{D}_{total}$} 
                        \STATE $\mathcal{L}\leftarrow \mathcal{L}(f^{(L^*)}_{\theta}(x), y)$
                        \STATE $\theta \leftarrow \theta - \eta \text{Optimizer}(\nabla_{\theta}\mathcal{L})$
                    \ENDFOR 
                \ENDFOR 
                \STATE 
                \STATE {\bfseries Return:} Robust and Calibrated $L^{*}$-Lipschitz network $f^{(L^*)}_{\theta}$.
            \end{algorithmic}
        \end{algorithm}

    % \section{Formal Proofs}
    % \label{app:formal_proofs}

    \section{Existence of \texorpdfstring{$L^*$}{L*}}
    \label{app:existence_of_the_optimal_lipschitz_constant}
    
        % \subsection{Existence of $L^*$}
        % \label{app:existence_of_the_optimal_lipschitz_constant}
        
        We provide a formal proof of the existence of an intrinsically calibrated Lipschitz constant $L^*$, i.e., a training constraint under which the network requires no post-hoc temperature rescaling ($T^* = 1$). The argument relies on the Intermediate Value Theorem applied to the optimal calibration temperature $T^*(L)$, viewed as a function of the training constraint $L$. To handle the proof cleanly, we work with the inverse temperature $\beta = 1/T$, for which the calibration objective is a standard log-linear model and admits a globally strictly convex formulation.
        
        Let $f^{(L)}_\theta$ denote a network trained to convergence under Lipschitz constraint $L > 0$, with logits $z^{(L)}(x) = f^{(L)}_\theta(x)$. The cross-entropy on the calibration set $\mathcal{D}_{\mathrm{cal}} = \{(x_i, y_i)\}_{i=1}^N$, reparameterized in $\beta = 1/T$, reads:
        \begin{equation}\label{eq:psi_def}
            \Psi(\beta, L) = \frac{1}{N} \sum_{i=1}^N \left[ -\beta\, z^{(L)}_{i, y_i} + \mathrm{LSE}\!\left(\beta\, z^{(L)}_i\right) \right],
        \end{equation}
        where $\mathrm{LSE}(u) = \log \sum_c e^{u_c}$ denotes the log-sum-exp function. The optimal inverse calibration temperature is
        \begin{equation}\label{eq:beta_star_def}
            \beta^*(L) = \arg\min_{\beta > 0} \Psi(\beta, L), 
            \qquad T^*(L) = 1/\beta^*(L).
        \end{equation}
        
        \begin{assumption}[Trained-network family]\label{ass:family}
            We assume access to a family $\{f^{(L)}_\theta\}_{L > 0}$ of networks, each trained under Lipschitz constraint $L$, satisfying the following regularity properties:
            \begin{enumerate}[label=(\roman*), leftmargin=*]
                \item For every $L > 0$, the logits $z^{(L)}(x) = f^{(L)}_\theta(x)$ are measurable in $x$ and uniformly bounded on every compact subset $K \subset (0, +\infty)$: there exists $B_K < +\infty$ such that $\|z^{(L)}(x)\|_\infty \leq B_K$ for all $L \in K$ and all $x \in \mathcal{X}$.
                \item For $\mathcal{D}$-almost every $x$, the map $L \mapsto z^{(L)}(x)$ is continuous on $(0, +\infty)$.
            \end{enumerate}
        \end{assumption}
        
        \paragraph{Discussion.}
        Assumption~\ref{ass:family} posits the existence of a continuous, well-behaved trained-network family. In practice, the SGD-based optimization of non-convex objectives offers no formal guarantee that $L \mapsto \theta^*(L)$ is continuous, since local minima may bifurcate or disappear as $L$ varies. Assumption~\ref{ass:family} should therefore be read as a regularity property of the training procedure itself, empirically supported by Figure~\ref{fig:training_dynamics}: across initializations spanning three orders of magnitude, the trained networks consistently converge to the same intrinsic $L^*$, suggesting that the realized family $L \mapsto f^{(L)}_\theta$ behaves smoothly in the regimes of practical interest. Item (i) is automatic when $\mathcal{X}$ is compact (Section~\ref{sec:background_and_related_work}) and the architecture has continuous activations.
        
        \begin{assumption}[Non-degeneracy of $\mathcal{D}_{\mathrm{cal}}$]\label{ass:nondegen}
            For every $L > 0$, the calibration set $\mathcal{D}_{\mathrm{cal}}$ is non-degenerate under 
            $f^{(L)}_\theta$, in the sense that:
            \begin{enumerate}[label=(\roman*), leftmargin=*]
                \item there exists at least one sample $(x_i, y_i) \in \mathcal{D}_{\mathrm{cal}}$ such that $z^{(L)}_{i, y_i} < \max_{c \neq y_i} z^{(L)}_{i, c}$ (at least one misclassified example);
                \item there exists at least one sample for which the logits $z^{(L)}_i$ are not all equal;
                \item the average margin is positive on $\mathcal{D}_{\mathrm{cal}}$:
                \begin{equation}\label{eq:positive_margin}
                    \frac{1}{N} \sum_{i=1}^N \left[ z^{(L)}_{i, y_i} - \frac{1}{C} \sum_c z^{(L)}_{i, c} \right] > 0.
                \end{equation}
            \end{enumerate}
        \end{assumption}
        
        Conditions (i) and (ii) are very mild: (i) ensures that $\beta = +\infty$ is not optimal (which would correspond to $T^* = 0$), and (ii) excludes the degenerate case of a constant predictor. Condition (iii) requires that the trained network performs strictly better than a uniform predictor on $\mathcal{D}_{\mathrm{cal}}$, which is automatically satisfied by any model trained beyond random initialization. It ensures that the cross-entropy on $\mathcal{D}_{\mathrm{cal}}$ admits an interior minimizer in $(0, +\infty)$, excluding $\beta = 0$ as the infimum.
        
        \begin{lemma}[Existence and uniqueness of $\beta^*(L)$]\label{lem:existence_beta}
            Under Assumptions~\ref{ass:family} and~\ref{ass:nondegen}, for every $L > 0$ the function $\beta \mapsto \Psi(\beta, L)$ is globally strictly convex on $(0, +\infty)$ and attains its infimum at a unique interior point $\beta^*(L) \in (0, +\infty)$.
        \end{lemma}
        
        \begin{proof}
            Fix $L > 0$ and write $z_i = z^{(L)}(x_i)$. We split the proof into strict convexity and the analysis of the boundary behavior.
            
            \paragraph{Strict convexity.}
            The objective $\Psi(\beta, L)$ is the negative log-likelihood of a multinomial logistic regression with fixed features $z_i$ and a single scalar parameter $\beta$. Differentiating twice with respect to $\beta$:
            \begin{equation}\label{eq:second_derivative}
                \frac{\partial^2 \Psi}{\partial \beta^2}(\beta, L) = \frac{1}{N} \sum_{i=1}^N \mathrm{Var}_{\hat{p}_i^{(\beta)}}[z_{i, c}],
            \end{equation}
            where $\hat{p}_i^{(\beta)} = \mathrm{softmax}(\beta z_i)$ and $\mathrm{Var}_{\hat{p}_i^{(\beta)}}[z_{i,c}] = \sum_c \hat{p}_{i,c}^{(\beta)} z_{i,c}^2 - \big(\sum_c \hat{p}_{i,c}^{(\beta)} z_{i,c}\big)^2 \geq 0$ is the variance of the logits under the softmax distribution. By Assumption~\ref{ass:nondegen}(ii), at least one sample has non-constant logits, so the corresponding variance is strictly positive, and $\partial^2 \Psi / \partial \beta^2 > 0$ for all $\beta > 0$. Hence $\Psi(\cdot, L)$ is globally strictly convex on $(0, +\infty)$.
            
            \paragraph{Existence of an interior minimizer.}
            We show that $\Psi(\cdot, L)$ attains its infimum at an interior point of $(0, +\infty)$ by establishing the boundary behavior.
            
            \emph{As $\beta \to +\infty$:} for each sample $i$, we have
            \begin{equation*}
                \mathrm{LSE}(\beta z_i) = \beta \max_c z_{i,c} + \log\!\Big( \sum_c \exp\!\big(\beta(z_{i,c} - \max_c z_{i,c})\big) \Big)= \beta \max_c z_{i,c} + o(1),
            \end{equation*}
            hence
            \begin{equation}\label{eq:large_beta}
                -\beta z_{i, y_i} + \mathrm{LSE}(\beta z_i) = \beta\!\left(\max_c z_{i,c} - z_{i, y_i}\right) + o(1).
            \end{equation}
            By Assumption~\ref{ass:nondegen}(i), at least one sample satisfies $\max_c z_{i,c} - z_{i, y_i} > 0$, so $\Psi(\beta, L) \to +\infty$ as $\beta \to +\infty$.
            
            \emph{As $\beta \to 0^+$:} the softmax tends to the uniform distribution, so $\Psi(\beta, L) \to \log C$, which is a \emph{finite} limit. We therefore cannot invoke coercivity at $0^+$. Instead, we examine the right-derivative at $\beta = 0$:
            \begin{equation}\label{eq:derivative_zero}
                \frac{\partial \Psi}{\partial \beta}(0^+, L) = \frac{1}{N} \sum_{i=1}^N \left[ \frac{1}{C} \sum_c z^{(L)}_{i, c} - z^{(L)}_{i, y_i} \right].
            \end{equation}
            By Assumption~\ref{ass:nondegen}(iii), this quantity is strictly negative. Hence, there exists $\beta_0 > 0$ such that $\Psi(\beta_0, L) < \log C = \lim_{\beta \to 0^+} \Psi(\beta, L)$.
            
            \paragraph{Conclusion.}
            Combining strict convexity, the divergence at $+\infty$, and the strict decrease in a right-neighborhood of $0$, the function $\Psi(\cdot, L)$ attains its infimum at a unique interior point $\beta^*(L) \in (0, +\infty)$.
        \end{proof}
        
        \begin{lemma}[Continuity of $\beta^*(L)$ and $T^*(L)$]\label{lem:continuity}
            Under Assumptions~\ref{ass:family} and~\ref{ass:nondegen}, the maps $L \mapsto \beta^*(L)$ and $L \mapsto T^*(L)$ are continuous on $(0, +\infty)$.
        \end{lemma}
        
        \begin{proof}
            We argue via Berge's Maximum Theorem. We first establish joint continuity of $\Psi$, then exhibit a uniform compact containing the minimizer on a neighborhood of any $L_0$.
        
            \paragraph{Joint continuity.}
            Since $\mathcal{D}_{\mathrm{cal}}$ is finite, $\Psi(\beta, L) = \frac{1}{N} \sum_{i=1}^N \big[ -\beta z^{(L)}_{i, y_i} + \mathrm{LSE}(\beta z^{(L)}_i) \big]$ is a finite sum of compositions of continuous functions: $(\beta, L) \mapsto \beta z^{(L)}_i$ is continuous by Assumption~\ref{ass:family}(ii), and $\mathrm{LSE}$ is continuous on $\mathbb{R}^C$. Hence $\Psi$ is jointly continuous on $(0, +\infty)^2$.
            
            \paragraph{Uniform compact.}
            Fix $L_0 > 0$ and let $V = [L_0/2, 2 L_0]$. By Assumption~\ref{ass:family}(i), there exists $B_V < +\infty$ such that $\|z^{(L)}(x)\|_\infty \leq B_V$ for all $L \in V$ and all $x \in \mathcal{X}$. The boundary estimates of Lemma~\ref{lem:existence_beta} can then be made uniform on $V$: there exist $0 < \beta_- < \beta_+ < +\infty$ such that
            \begin{equation*}
                \beta^*(L) \in [\beta_-, \beta_+] \quad \text{for all } L \in V.
            \end{equation*}
            Indeed, equation~\eqref{eq:large_beta} combined with the uniform bound on the margins (deduced from $B_V$ and Assumption~\ref{ass:nondegen}(i)) yields a $\beta_+$ independent of $L \in V$ beyond which $\Psi(\beta, L) > \log C$; similarly, equation~\eqref{eq:derivative_zero} combined with a uniform positive lower bound on the average margin (Assumption~\ref{ass:nondegen}(iii) applied uniformly on $V$ via $B_V$) yields a $\beta_-$ below which $\partial \Psi / \partial \beta < 0$ for all $L \in V$. Set $K = [\beta_-, \beta_+]$.
            
            \paragraph{Conclusion.}
            The constant correspondence $\Gamma : V \to K$ is trivially continuous (both upper and lower hemicontinuous). By Berge's Maximum Theorem applied to $(\beta, L) \mapsto \Psi(\beta, L)$ on the compact-valued correspondence $\Gamma$, the argmin correspondence
            \begin{equation}\label{eq:argmin_correspondence}
                L \mapsto \arg\min_{\beta \in K} \Psi(\beta, L)
            \end{equation}
            is upper hemicontinuous. By Lemma~\ref{lem:existence_beta}, the minimizer is unique on $(0, +\infty)$, and \emph{a fortiori} on $K$, so this correspondence is single-valued and upper hemicontinuity reduces to continuity of $L \mapsto \beta^*(L)$ on $V$. Since $L_0 \in (0, +\infty)$ was arbitrary, $\beta^*$ is continuous on $(0, +\infty)$.
            
            Finally, $T^*(L) = 1/\beta^*(L)$ is continuous on $(0, +\infty)$ as the composition of $\beta^*$ and the continuous map $\beta \mapsto 1/\beta$ on $(0, +\infty)$, well-defined since $\beta^*(L) > 0$ by Lemma~\ref{lem:existence_beta}.
        \end{proof}
        
        \begin{assumption}[Under-confidence for small $L$]\label{ass:underconf}
            There exists $L_- > 0$ such that $T^*(L_-) < 1$, equivalently $\beta^*(L_-) > 1$. Mechanistically, as $L \to 0$, logit magnitudes are bounded by $L \cdot \mathrm{diam}(\mathcal{X}) \to 0$, forcing the softmax distribution towards uniform and yielding systematically low-confidence predictions. This regime is empirically confirmed in Figures~\ref{fig:correlation},~\ref{fig:pareto_analysis}, and~\ref{fig:fixed_reliability_diagram}.
        \end{assumption}
        
        \begin{assumption}[Over-confidence for large $L$]\label{ass:overconf}
            There exists $L_+ > L_-$ such that $T^*(L_+) > 1$, equivalently $\beta^*(L_+) < 1$. For $L$ sufficiently large, the Lipschitz constraint becomes inactive and the network behaves as an unconstrained model trained with cross-entropy, which is known to be systematically over-confident \citep{guo_calibration_2017}. This regime is empirically confirmed in Figure~\ref{fig:pareto_analysis} (left) and Figure~\ref{fig:fixed_reliability_diagram}.
        \end{assumption}
        
        \begin{proposition}[Existence of $L^*$]\label{prop:existence_L_star}
            Under Assumptions~\ref{ass:family},~\ref{ass:nondegen},~\ref{ass:underconf}, and~\ref{ass:overconf}, there exists $L^* \in (0, +\infty)$ such that
            \begin{equation}\label{eq:L_star_def}
                T^*(L^*) = 1,
            \end{equation}
            i.e., the network trained under constraint $L^*$ is intrinsically calibrated and requires no post-hoc temperature rescaling.
        \end{proposition}
        
        \begin{proof}
            Let $L_-, L_+$ be given by Assumptions~\ref{ass:underconf} and~\ref{ass:overconf} respectively, with $0 < L_- < L_+$. By Lemma~\ref{lem:continuity}, $T^*$ is continuous on $(0, +\infty)$, hence on $[L_-, L_+]$. Since $T^*(L_-) < 1 < T^*(L_+)$, the Intermediate Value Theorem yields the existence of at least one $L^* \in (L_-, L_+)$ such that $T^*(L^*) = 1$.
        \end{proof}
        
        \begin{remark}
            Proposition~\ref{prop:existence_L_star} guarantees existence but not uniqueness of $L^*$. In practice, Figure~\ref{fig:correlation} suggests that the calibration profile $L \mapsto \mathrm{ECE}(L)$ has a single well-defined minimum on the datasets considered, consistent with a unique crossing of $T^*(L) = 1$. A sufficient condition for uniqueness would be the strict monotonicity of $T^*(L)$ in $L$, which we conjecture holds under mild regularity conditions on the training dynamics, and leave as an open theoretical question.
        \end{remark}
        
        \begin{remark}
            Assumptions~\ref{ass:underconf} and~\ref{ass:overconf} are not merely technical: they identify the two boundary regimes that bracket $L^*$. Both are empirically validated (Figures~\ref{fig:correlation},~\ref{fig:pareto_analysis},~\ref{fig:fixed_reliability_diagram}) and structurally motivated: under-confidence for small $L$ follows from the contraction of logit magnitudes (Section~\ref{sec:the_lipschitz_temperature_duality}), while over-confidence for large $L$ is the well-documented behavior of unconstrained networks trained with cross-entropy \citep{guo_calibration_2017}. Together with Assumption~\ref{ass:family}, which posits the existence of a continuous trained-network family, they form a set of regularity conditions on the training procedure rather than a derivation from first principles. The empirical convergence of LiST to $L^*$ across initializations (Figure~\ref{fig:training_dynamics}) constitutes the strongest \emph{a posteriori} validation of these assumptions.
        \end{remark}
        
        \begin{remark}[Reparameterization in $\beta$]
            The reparameterization $\beta = 1/T$ is mathematically convenient because $\Psi(\beta, L)$ is the negative log-likelihood of a one-parameter log-linear model in $\beta$, which is a standard globally strictly convex objective \citep{boyd_convex_2023}. Working directly in $T$, by contrast, yields an objective $\Phi(T, L) = \Psi(1/T, L)$ whose second derivative in $T$ involves both the variance term and a first-order chain-rule term that does not have a fixed sign globally. The two formulations are equivalent at the optimum (where $T^* = 1/\beta^*$), but the $\beta$ formulation makes the strict convexity argument transparent.
        \end{remark}

    \section{Additional Results and Qualitative Analysis}
    \label{app:additional_results_and_qualitative_analysis}

        \subsection{Additional Results on CIFAR-100}
        \label{app:additional_results_on_cifar100}

        \Cref{fig:correlation_cifar100} and~\Cref{fig:pareto_analysis_cifar100} extends the analyses of~\Cref{fig:correlation} and~\Cref{fig:pareto_analysis} to CIFAR-100, a significantly harder classification task with 100 classes. The ECE profile as a function of the Lipschitz constant (\Cref{fig:correlation_cifar100}a) reproduces the qualitative pattern observed on CIFAR-10: a low constraint yields systematic underconfidence, while a high constraint leads to overconfidence, with a well-defined minimum in between. This confirms that the existence of a calibrated operating point $L^*$ is not specific to CIFAR-10.
        \begin{figure}[t]
            \centering
            \includegraphics[width=\linewidth]{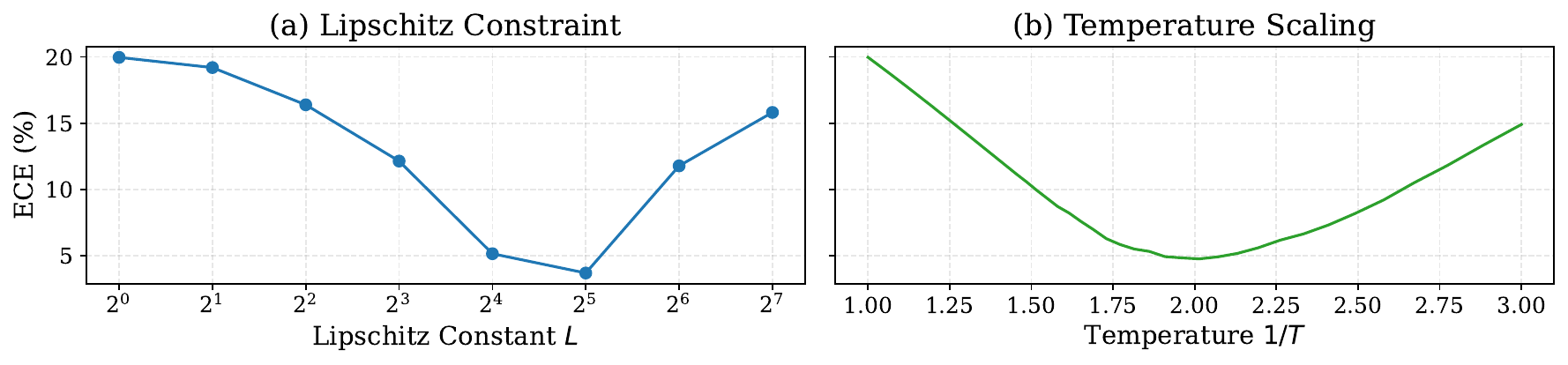}
            \caption{
                \textbf{ECE as a function of the logit scaling mechanism.} (a) Networks trained with varying fixed Lipschitz constants $L$. (b) Post-hoc Temperature Scaling on a fixed 1-Lipschitz network (plotted against $1/T$). Both profiles are qualitatively similar.
            }
            \label{fig:correlation_cifar100}
        \end{figure}

        The Pareto analysis of~\Cref{fig:pareto_analysis_cifar100} further corroborates the findings of~\Cref{subsec:results}. With $\xi = 0$, LiST converges to a natural equilibrium on the accuracy-robustness Pareto front, consistent with the calibrated point identified on CIFAR-10. Varying the offset $\xi$ recovers a fully calibrated Pareto front, demonstrating that the decoupling between calibration and the accuracy-robustness trade-off generalizes beyond CIFAR-10.
        \begin{figure}[t]
            \centering
            \includegraphics[width=\linewidth]{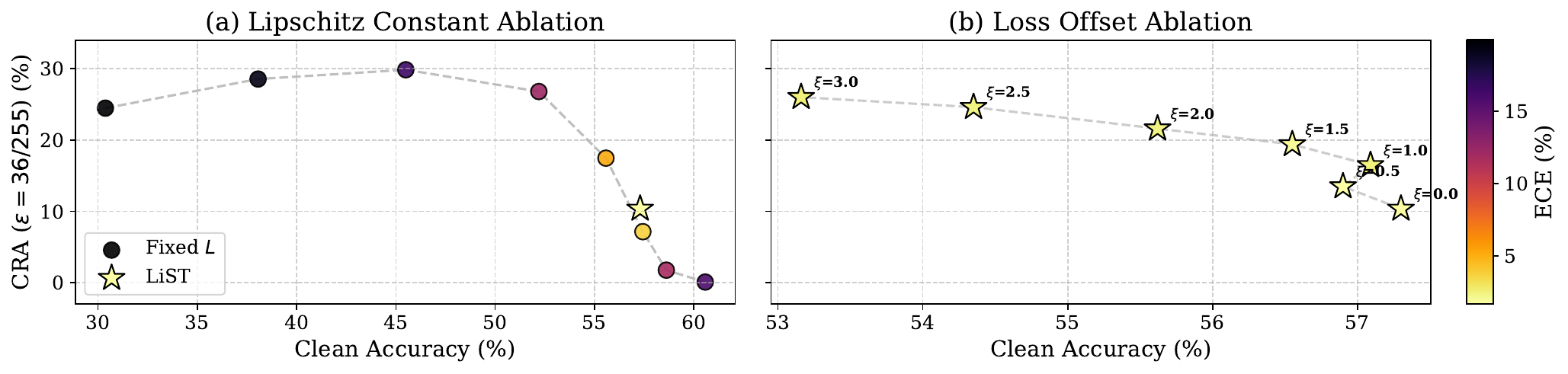}
            \caption{\textbf{Pareto Analysis on CIFAR-100.} (Left) LiST naturally converges to the calibrated point on the Pareto front. (Right) Varying $\xi$ yields a fully calibrated Pareto front using $100\%$ of training data.}
            \label{fig:pareto_analysis_cifar100}
        \end{figure}

        \subsection{Qualitative Analysis}
        \label{app:qualitative_analysis}

        To complement the quantitative metrics reported in the main paper, we provide reliability diagrams~\citep{murphy_reliability_1977} to assess calibration visually. The visual style is inspired by~\citep{vasilev_calibration_2023}, which we found particularly informative.

        \Cref{fig:fixed_reliability_diagram} shows the reliability diagrams of the fixed-$L$ baselines on CIFAR-10 and CIFAR-100, corresponding to the networks evaluated in~\Cref{fig:correlation}, \ref{fig:pareto_analysis}, \ref{fig:correlation_cifar100}, \ref{fig:pareto_analysis_cifar100} and~\Cref{tab:main_results}. On both datasets, the diagrams confirm the systematic transition from underconfidence ($L \rightarrow 0$) to overconfidence ($L \rightarrow \infty$), with a calibrated regime in between. This visually corroborates the ESCE results of~\Cref{tab:main_results} and provides empirical support for Lemma~\ref{lem:continuity} and Proposition~\ref{prop:existence_L_star}.
        \begin{figure}[t]
            \centering
            \begin{subfigure}{\linewidth}
                \includegraphics[width=\linewidth]{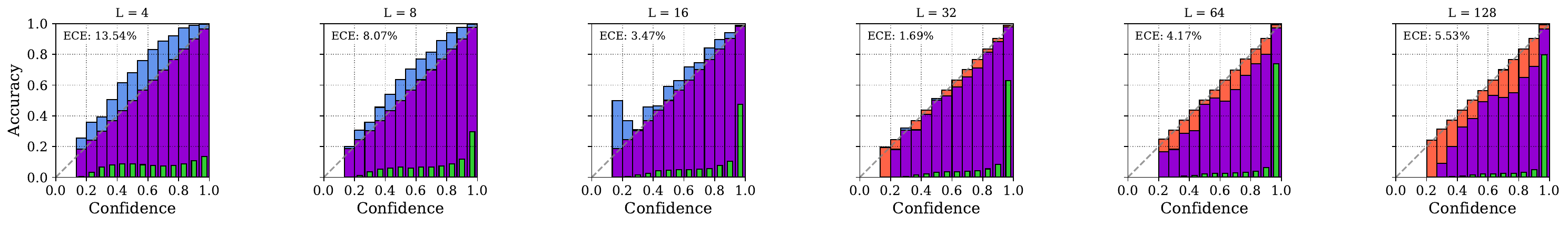}
                \caption{CIFAR-10}
            \end{subfigure}
            \hfill
            \begin{subfigure}{\linewidth}
                \includegraphics[width=\linewidth]{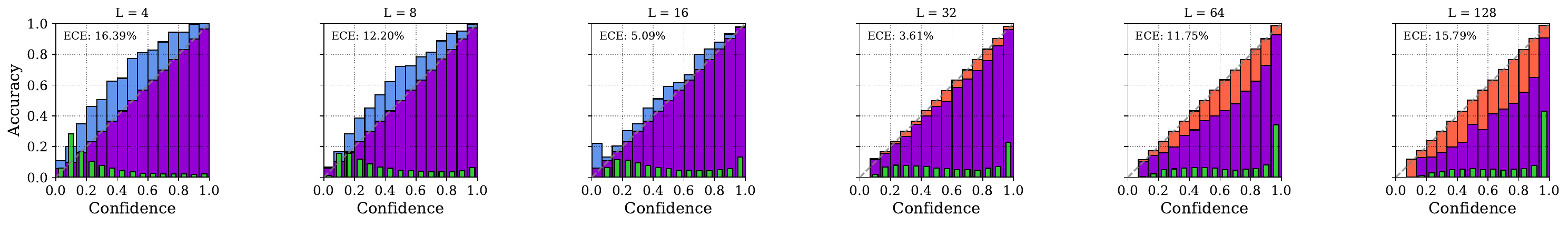}
                \caption{CIFAR-100}
            \end{subfigure}
            \caption{\textbf{Reliability Diagrams of the fixed-$L$ baselines on CIFAR.} On both datasets, we observe the same trend of under-to-over confidence as $L$ grows.}
            \label{fig:fixed_reliability_diagram}
        \end{figure}
        \Cref{fig:cross_datasets_reliability} shows the reliability diagrams of LiST on CIFAR-10, CIFAR-100 and Tiny-ImageNet (with $\xi = 0$). Across all three datasets, LiST produces diagrams that closely follow the diagonal, qualitatively confirming that the networks are intrinsically calibrated.
        \begin{figure}[t]
            \centering
            \includegraphics[width=0.7\linewidth]{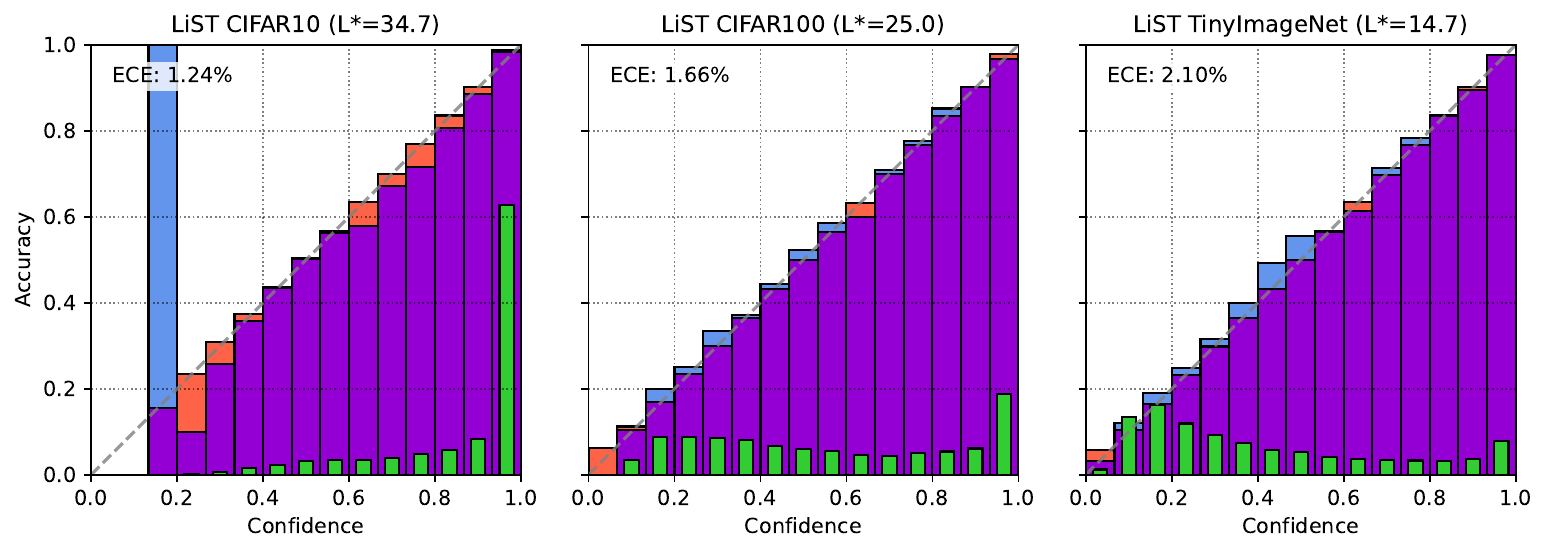}
            \caption{\textbf{Reliability Diagrams of LiST for each dataset.} LiST networks are intrinsically calibrated on CIFAR-10/100 and Tiny-ImageNet.}
            \label{fig:cross_datasets_reliability}
        \end{figure}

    \section{More Related Works}
    \label{app:more_related_works}

        \subsection{Calibration of Neural Networks}
        \label{app:calibration_of_neural_networks}

            The study of calibration has produced a rich body of work, which we briefly survey here to contextualize LiST. \textit{Post-hoc methods} recalibrate a fixed model on a held-out set: beyond Temperature Scaling~\citep{guo_calibration_2017}, Vector and Matrix Scaling~\citep{guo_calibration_2017} introduce class-wise rescaling, while Dirichlet Calibration~\citep{kull_beyond_2019} generalizes this to a full Dirichlet reparametrization. Histogram Binning~\citep{zadrozny_obtaining_2001} and Isotonic Regression~\citep{zadrozny_transforming_2002} are non-parametric alternatives, and Bayesian Binning into Quantiles~\citep{naeini_obtaining_2015} extends them with a Bayesian treatment. While effective on i.i.d. test data, these methods leave the underlying model untouched and therefore cannot improve robustness, as we discuss in~\Cref{subsec:the_robustness_accuracy_calibration_trade_off}.
    
            \textit{Training-time methods} instead modify the loss to encourage calibration during optimization. Label Smoothing~\citep{muller_when_2019} softens one-hot targets and reduces overconfidence, while Focal Loss~\citep{mukhoti_calibrating_2020} downweights well-classified examples and was shown to improve calibration as a side effect. MMCE~\citep{kumar_trainable_2018} introduces a kernel-based calibration regularizer, and AvUC~\citep{krishnan_improving_2020} adds an accuracy-versus-uncertainty term. More recently, Dual Focal Loss~\citep{tao_dual_2023} and MbLS~\citep{liu2022devil} further refine the loss design for improved calibration. Unlike LiST, none of these approaches leverages the model's own structural properties as a feedback signal during training, and none provides robustness guarantees.
    
            A complementary line of work studies the \textit{interaction between calibration and robustness}.~\citep{grabinski_robust_2022} show that adversarially trained models tend to be less overconfident, and~\citep{qin_improving_2021} propose AR-AdaLS to improve calibration using per-sample robustness. Most relevant to our work,~\citep{grosse_limitations_2019} and~\citep{galil_disrupting_2021} highlight that uncertainty-targeted attacks can break calibration even when accuracy is preserved, motivating training procedures that produce intrinsically calibrated robust models, a property that LiST achieves by construction.

        \subsection{Lipschitz-Constrained Neural Networks}
        \label{app:lipschitz_constrained_neural_networks}

            Lipschitz-constrained architectures have emerged as a principled route to certified robustness. Early work focused on enforcing the constraint via spectral normalization~\citep{miyato_spectral_2018} or via orthogonal parameterizations~\citep{anil_sorting_2019}. Subsequent contributions improved the expressivity of these architectures: Cayley convolutions~\citep{trockmanorthogonalizing}, BCOP~\citep{li_preventing_2019}, SOC~\citep{singla_skew_2021}, LOT~\citep{xu_lot_2022}, and AOL~\citep{prach_almost-orthogonal_2023}. More recent efforts, including SLL~\citep{araujo2023unified}, CPL~\citep{meunier_dynamical_2022}, and the adaptive orthogonal convolutions of~\citep{boissin2025adaptive} further close the gap between expressivity and certifiability. Our experiments rely on this last family of architectures, but LiST is agnostic to the specific Lipschitz parametrization and could be combined with any of the above.

            A second axis concern the \textit{training objectives} in conjunction with the Lipschitz constraint.~\citep{tsuzuku_lipschitz-margin_2018} introduced margin-based training to improve certified radii, while~\citep{serrurier_achieving_2021} proposed a hinge-regularized optimal transport loss.~\citep{bethune_pay_2022} showed that, under fixed $L=1$, the temperature parameter $\tau$ in the cross-entropy loss governs the entire accuracy-robustness Pareto front, and recent works~\citep{prach_almost-orthogonal_2023, prach2025intriguing, lai2025enhancing} further refine this loss family. Our analysis (\Cref{sec:the_lipschitz_temperature_duality}) reveals that the $\tau$ parameter and the Lipschitz constant $L$ play structurally equivalent roles on the logits, which motivates absorbing $\tau$ into $L$ and treating $L$ itself as the dynamic control variable; a perspective, to our knowledge, not previously explored.

            Finally, a few works have studied calibration in Lipschitz networks implicitly.~\citep{ye_mitigating_2023} show that enforcing Lipschitz regularization on Transformer architectures mitigates overconfidence, corroborating our observation that the Lipschitz constant structurally governs the confidence level of a network. However, their approach uses regularization as a soft penalty rather than as a dynamic control variable, and does not establish a formal connection to Temperature Scaling nor provide certified robustness guarantees. To the best of our knowledge, LiST is the first method to use calibration as an explicit training signal for Lipschitz-constrained networks,
    
    \section{Implementation Details and Compute}
    \label{app:implementation}
    
        \paragraph{Libraries and Layer Details.}
        All Lipschitz-constrained networks are implemented using the \texttt{deel-torchlip}\footnote{\url{https://github.com/deel-ai/deel-torchlip}} and \texttt{orthogonium}\footnote{\url{https://github.com/thib-s/orthogonium}} libraries, which provide certified orthogonal convolutional and linear layers. \texttt{BatchCentering} replaces \texttt{BatchNormalization} at the same positions within the network, performing only a centering operation (zero-mean normalization without rescaling) to preserve the Lipschitz bound throughout the architecture.

        \paragraph{Optimizer.} 
        All experiments (LiST and baselines) use AdamW-ScheduleFree~\citep{defazio_road_2024} as the optimizer, which removes the need for a learning rate schedule via an internal weight-averaging mechanism. This choice ensures a fair comparison across methods and avoids confounding the results with schedule-specific tuning.

        \paragraph{Training budgets.}
        Each unconstrained baseline is trained for the number of epochs typically reported as sufficient to reach convergence in its original setting: 100 epochs for Standard cross-entropy and Label Smoothing, 350 epochs for Focal Loss (as in~\citep{mukhoti_calibrating_2020}), and 200 epochs for PGD-AT and CAAT. For the fixed-L Lipschitz baselines, we match the number of epochs required by LiST to converge on the same dataset (e.g., approximately 2200 epochs on CIFAR-10, see Stopping Criterion Sensitivity below), ensuring that fixed-$L$ models are not disadvantaged by an artificially shorter training budget.
        
        \paragraph{Stopping Criterion Sensitivity}
        The stopping criterion halts training when the relative range of $L_t$ over a sliding window of $W = 30$ epochs falls below $\varepsilon = 0.001$, corresponding to a variation of less than $0.1\%$ over the window. As a reference point, on CIFAR-10 this criterion triggers convergence after approximately 2200 epochs. These values were found to be robust across the datasets considered and were not individually tuned per dataset.

        \paragraph{Statistical Significance.}
        All results on CIFAR-10 reported in this paper are averaged over 3 independent runs with different random seeds, and we report mean and standard deviation throughout. The variance across runs is low in all settings: the standard deviation on clean accuracy remains below 0.30\%, on ECE below 0.25\%, and on CRA below 1.76\% for LiST, and even lower for fixed-L baselines. These figures confirm that the observed gains of LiST over competing methods (in particular the improvement in ECE) are not attributable to random fluctuations. Results on CIFAR-100 and Tiny-ImageNet are reported for a single run; given the low variance observed on CIFAR-10, which indicates stable training dynamics across seeds, and the prohibitive computational cost of repeating full training runs on larger datasets, we consider a single run sufficient to draw reliable conclusions.
        
        \paragraph{Compute Resources.}
        All experiments were conducted on a single NVIDIA RTX~4090 GPU. A full LiST training run on each dataset takes approximately 15 hours.
        
        We note that a non-negligible fraction of this wall-clock cost stems from evaluating test metrics at every epoch. Since ScheduleFree~\cite{defazio_road_2024} requires an additional weight-averaging pass, and BatchCentering statistics must be stabilized over ${\sim}50$ batches, each evaluation step involves overhead beyond the standard training forward pass. A lighter evaluation cadence (e.g., every 10 epochs) would reduce training time significantly without affecting convergence.
        
        The experiments reported in this paper (across all datasets, baselines, and ablations) correspond to approximately 100 training runs in total. Accounting for preliminary experiments and failed runs during development, the total compute budget is estimated at around 300 runs.

        \paragraph{Model Definition.}
        We show in~\Cref{tab:lipschitz_resnet18} the ResNet-18 Lipschitz architecture used in this paper. The $\sigma$ scaling factor is modified at each epoch to ensure that $\sigma^{18} = L_e$. It is thus computed as $\sigma = L_e^{1/18}$ and distributed uniformly across all layers (dense and convolutions). 

        \begin{table}[t]
            \centering
            \caption{
                \textbf{Architecture of the Lipschitz-constrained ResNet-18 used in this work.} Each BasicBlock follows the structure Conv$3{\times}3 \rightarrow$ BC $\rightarrow$ GroupSort2 $\rightarrow$ Conv$3{\times}3$, combined with a residual skip connection via a convex blend $\alpha \oplus (1-\alpha)\cdot\text{skip}$. BC denotes BatchCentering. The Lipschitz factor of each module is reported in the last column, with $\sigma$ the spectral bound of the corresponding orthogonal layer.
            }
            \label{tab:lipschitz_resnet18}
            \small
            \begin{tabular}{@{}lll@{}}
                \toprule
                \textbf{Module} & \textbf{Output shape} & \textbf{Lip. factor} \\
                \midrule
                \makecell[l]{\textbf{Stem} \\ \quad AdaptiveOrthoConv2d $\cdot$ BatchCentering $\cdot$ GroupSort2}
                  & $64 \times 32 \times 32$ & $\sigma$ \\
                \addlinespace
                \makecell[l]{\textbf{Stage 1} $\cdot$ 2 $\times$ BasicBlock $\cdot$ 64 ch \\
                             \quad skip: Identity, \quad stride 1}
                  & $64 \times 32 \times 32$ & $\sigma \times 4$ \\
                \addlinespace
                \makecell[l]{\textbf{Stage 2} $\cdot$ 2 $\times$ BasicBlock $\cdot$ 128 ch \\
                             \quad skip: PixelUnshuffle(2) + Conv$1{\times}1$, \quad stride 2}
                  & $128 \times 16 \times 16$ & $\sigma \times 4$ \\
                \addlinespace
                \makecell[l]{\textbf{Stage 3} $\cdot$ 2 $\times$ BasicBlock $\cdot$ 256 ch \\
                             \quad skip: PixelUnshuffle(2) + Conv$1{\times}1$, \quad stride 2}
                  & $256 \times 8 \times 8$ & $\sigma \times 4$ \\
                \addlinespace
                \makecell[l]{\textbf{Stage 4} $\cdot$ 2 $\times$ BasicBlock $\cdot$ 512 ch \\
                             \quad skip: PixelUnshuffle(2) + Conv$1{\times}1$, \quad stride 2}
                  & $512 \times 4 \times 4$ & $\sigma \times 4$ \\
                \addlinespace
                ScaledAdaptiveL2NormPool2d ($1{\times}1$) & $512$ & --- \\
                \addlinespace
                SpectralLinear $\rightarrow C$ classes & $C$ & $\sigma$ \\
                \bottomrule
            \end{tabular}
        \end{table}

    \section{Broader Impacts}
    \label{sec:broader_impacts}
    
        LiST is a generic training procedure that improves the robustness and
        calibration of image classifiers. Better-calibrated and more robust models support safer deployment in high-stakes settings (e.g.\ medical imaging or autonomous perception). Conversely, the same properties could be misused to lend unwarranted credibility to models trained on biased data or deployed in surveillance applications; LiST does not address dataset bias or the legitimacy of the underlying task. Our experiments use only standard public benchmarks, and we recommend combining LiST with fairness auditing and human oversight before any sensitive deployment.
    
    %\newpage
    %\input{checklist.tex}
    
\end{document}